\documentclass{article}

\usepackage{times}
\usepackage{graphicx} \usepackage{subfigure} 

\usepackage{natbib}

\usepackage{algorithm}
\usepackage{algorithmic}

\usepackage{amsmath}
\usepackage{amssymb}
\usepackage{amsthm}
\usepackage{amsfonts} \usepackage{bm}
\usepackage{setspace}
\usepackage{xspace}
\usepackage{scalerel} 
\usepackage{xcolor} \usepackage{wrapfig}
\usepackage{array} \usepackage{enumitem}
\usepackage{mdframed}

\usepackage[accepted]{icml2017}

\usepackage[colorinlistoftodos, textwidth=20mm, disable]{todonotes}
\definecolor{blued}{RGB}{70,197,221}

\newcommand{\todoaout}[1]{\todo[color=yellow]{\tiny#1}}
\definecolor{citrine}{rgb}{0.89, 0.82, 0.04}

\newcommand{\todomi}[1]{\todo[color=citrine, inline]{#1}}
\newcommand{\todomiout}[1]{\todo[color=citrine]{\tiny#1}}

\makeatletter{}
\renewcommand{\Re}{\mathbb{R}}

\newcommand{\feasibleset}{\mathcal{S}}
\newcommand{\coldict}{\mathcal{I}}
\newcommand{\pvec}{\mathbf{v}}
\newcommand{\pmat}{\mathbf{V}}

\newcommand{\onlinesqueak}{\textsc{KORS}\xspace}
\newcommand{\exactkons}{\textsc{KONS}\xspace}
\newcommand{\sketchkons}{\textsc{Sketched-KONS}\xspace}

\newtheorem{assumption}{Assumption}
\newtheorem{theorem}{Theorem}
\newtheorem{corollary}{Corollary}
\newtheorem{definition}{Definition}
\newtheorem{lemma}{Lemma}
\newtheorem{proposition}{Proposition}

\def\:#1{\protect \ifmmode {\mathbf{#1}} \else {\textbf{#1}} \fi}
\newcommand{\CommaBin}{\mathbin{\raisebox{0.5ex}{,}}}

\newcommand{\wt}[1]{\widetilde{#1}}
\newcommand{\wh}[1]{\widehat{#1}}

\newcommand{\wb}[1]{\overline{#1}}

\newcommand{\bsym}[1]{\mathbf{#1}}
\newcommand{\bzero}{\mathbf{0}}

\newcommand{\bb}{\mathbf{b}}

\newcommand{\be}{\mathbf{e}}
\newcommand{\bg}{\mathbf{g}}

\newcommand{\bk}{\mathbf{k}}

\newcommand{\br}{\mathbf{r}}

\newcommand{\bu}{\mathbf{u}}
\newcommand{\bv}{\mathbf{v}}
\newcommand{\bw}{\mathbf{w}}
\newcommand{\bx}{\mathbf{x}}

\newcommand{\bA}{\mathbf{A}}

\newcommand{\bD}{\mathbf{D}}
\newcommand{\bE}{\mathbf{E}}

\newcommand{\bH}{\mathbf{H}}
\newcommand{\bI}{\mathbf{I}}

\newcommand{\bK}{\mathbf{K}}

\newcommand{\bP}{\mathbf{P}}

\newcommand{\bR}{\mathbf{R}}
\newcommand{\bS}{\mathbf{S}}

\newcommand{\bU}{\mathbf{U}}
\newcommand{\bV}{\mathbf{V}}
\newcommand{\bW}{\mathbf{W}}
\newcommand{\bX}{\mathbf{X}}
\newcommand{\bY}{\mathbf{Y}}

\DeclareMathOperator*{\argmin}{arg\,min}

\renewcommand{\epsilon}{\varepsilon}
\newcommand{\bigotime}{\mathcal{O}}
\DeclareMathOperator*{\polylog}{polylog}

\newcommand{\normsmall}[1]{\Vert #1 \Vert}
\newcommand{\normempty}[1]{\left\Vert #1 \right\Vert}

\newcommand{\transp}{{\scaleobj{.8}{\mathsf{T}}}}
\DeclareMathOperator*{\Tr}{Tr}
\DeclareMathOperator*{\Det}{Det}

\DeclareMathOperator*{\Diag}{Diag}

\newcommand{\probability}{\mathbb{P}}

\DeclareMathOperator*{\expectedvalue}{\mathbb{E}}

\newcommand{\condbar}{\;\middle|\;}

\newcommand{\bernoullidist}{\mathcal{B}}
\newcommand{\filtration}{\mathcal{F}}

\newcommand{\indfunc}{\mathbb{I}}
\newcommand{\Real}{\mathbb{R}}

\newcommand{\statespace}{\mathcal{X}}

\newcommand{\dataset}{\mathcal{D}}

\newcommand{\rkhs}{\mathcal{H}}
\newcommand{\func}{f}

\newcommand{\hilbprod}[2]{\left\langle{#1},{#2}\right\rangle_\rkhs}

\newcommand{\kerfunc}{\mathcal{K}}
\newcommand{\kermatrix}{{\bK}}
\newcommand{\featkermatrix}{\boldsymbol{\Phi}}
\newcommand{\featvec}{\boldsymbol{\phi}}
\newcommand{\deff}{d_{\text{eff}}}
\newcommand{\donl}{d_{\text{onl}}}
\newcommand{\bdeff}{\wb{d}_{\text{eff}}}
\newcommand{\bdonl}{\wb{d}_{\text{onl}}}

\newcommand{\atau}{\wt{\tau}}

\newcommand{\akermatrix}{\wt{\mathbf{K}}}
\newcommand{\bkermatrix}{\mathbf{\wb{K}}}
\newcommand{\selmatrix}{{\bS}}
\newcommand{\selrmatrix}{{\bR}}

\PassOptionsToPackage{colorlinks,
           linkcolor=red,
           citecolor=blue,
           urlcolor=magenta,
           linktocpage,
           plainpages=false}{hyperref}

\icmltitlerunning{Second-Order Kernel Online Convex Optimization with Adaptive Sketching}

\begin{document} 

\twocolumn[
\icmltitle{Second-Order Kernel Online Convex Optimization with Adaptive Sketching}

\icmlsetsymbol{equal}{*}

\begin{icmlauthorlist}
\icmlauthor{Daniele Calandriello}{sequel}
\icmlauthor{Alessandro Lazaric}{sequel}
\icmlauthor{Michal Valko}{sequel}
\end{icmlauthorlist}

\icmlaffiliation{sequel}{SequeL team, INRIA Lille - Nord Europe}

\icmlcorrespondingauthor{Daniele Calandriello}{daniele.calandriello@inria.fr}

\icmlkeywords{kernels, online learning, sketching, machine learning, ICML}

\vskip 0.3in
]

\printAffiliationsAndNotice{}  
\begin{abstract} 
\emph{Kernel online convex optimization }(KOCO) is a framework combining the expressiveness
of non-parametric kernel models with the regret guarantees of online learning.
First-order KOCO methods such as functional gradient descent
require only $\bigotime(t)$ time and space per iteration,
and, when the only information on the losses is their convexity,
achieve a minimax optimal $\bigotime(\sqrt{T})$ regret. Nonetheless,
many common losses in kernel problems, such as squared loss, logistic loss,
and squared hinge loss posses stronger curvature that can be exploited.
In this case, second-order KOCO methods achieve $\bigotime(\log(\Det(\kermatrix)))$ regret,
which we show scales as $\bigotime(\deff\log T)$, where $\deff$
is the effective dimension of the problem and is usually much smaller than
$\bigotime(\sqrt{T})$. The main drawback of second-order methods
is their much higher $\bigotime(t^2)$ space and time complexity.
In this paper, we introduce \emph{kernel online Newton step} (KONS), a new second-order
KOCO method that also achieves $\bigotime(\deff\log T)$ regret. To address
the computational complexity of second-order methods, we introduce a new
matrix sketching algorithm for the kernel matrix~$\kermatrix_t$, and show
that for a chosen parameter $\gamma \leq 1$ our Sketched-KONS
reduces the space and time complexity by a factor of $\gamma^2$ to
$\bigotime(t^2\gamma^2)$ space and time per iteration, while
incurring only $1/\gamma$ times more regret.
\end{abstract}
\makeatletter{}
\section{Introduction}\label{sec:intro}
\emph{Online convex optimization} (OCO)~\citep{zinkevich2003online} models the problem
of convex optimization over $\Real^d$ as a game over  $t \in \{1,\dots,T\}$
time steps between an adversary  and the player. 
In its linear version, that we refer to as linear-OCO (LOCO), the adversary
chooses
a sequence of arbitrary convex losses $\ell_t$ and points 
$\bx_t$, and a player chooses weights~$\bw_t$ and predicts $\bx_t^\transp\bw_t$.
The goal of the player is to minimize the regret, 
defined as the difference between the losses of the predictions obtained using the weights
played by the player and the best fixed weight in hindsight given all points and losses.

\textbf{Gradient descent.} For this setting, \citet{zinkevich2003online} showed that simple \emph{gradient descent} (GD),
combined with a smart choice for the
stepsize $\eta_t$ of the gradient updates,
achieves a $\bigotime(\sqrt{dT})$ regret with a $\bigotime(d)$ space and time cost per iteration. When the only assumption on the losses
is simple convexity, this upper bound matches
the corresponding lower bound \cite{luo_efficient_2016}, thus making
first-order methods (e.g.,\,GD) essentially unimprovable in a minimax sense.
Nonetheless, when the losses have additional curvature properties,
\citet{hazan_logarithmic_2006} show that \emph{online Newton step} (ONS),
an adaptive method that exploits second-order (second derivative) information on the losses, can achieve a \emph{logarithmic}
regret $\bigotime(d\log T)$. The downside of this adaptive method is the
larger $\bigotime(d^2)$ space and per-step time complexity, since
second-order updates require to
construct, store, and invert~$\bH_t$, a preconditioner matrix related to the Hessian of the losses used to correct the first-order updates.

\textbf{Kernel gradient descent.}
For linear models, such as the ones considered
in LOCO, a simple way to create more expressive models is to map them in
some high-dimensional space, the \emph{feature space}, and then use the
\emph{kernel trick} \cite{scholkopf2001learning} to avoid explicitly computing their
high-dimensional representation. Mapping to a larger space allows the
algorithm to better fit the losses chosen by the adversary and reduce its cumulative loss.
As a drawback, the
Kernel OCO (KOCO) problem\footnote{This setting is often referred to as \textit{online kernel learning} or \textit{kernel-based online learning} in the literature.} is fundamentally harder than LOCO, due to 1) the fact that
an infinite parametrization makes regret bounds scaling with the dimension $d$
meaningless and 2) the size of the model, and therefore time and space complexities, scales with $t$
itself, making these methods even less performant than LOCO algorithms.
Kernel extensions of LOCO algorithms
have been proposed for KOCO, such as functional GD
(e.g., NORMA, \begin{NoHyper}\citealp{kivinen_online_2004}\end{NoHyper}) which achieves a $\bigotime(\sqrt{T})$
regret with a $\bigotime(t)$ space and time cost per iteration. For second-order
methods, the Second-Order Perceptron \cite{cesa2005second} or NAROW \cite{orabona2010new}
for generic curved losses and Recursive
Kernel Least Squares \cite{zhdanov_identity_2010} or Kernel AAR \cite{gammerman_2004}
for the specific case of $\ell_2$ losses provide bounds that scale with the
log-determinant of the kernel-matrix. As we 
show, this quantity is closely related
to the effective dimension $\deff^T$ of the of the points $\: x_t$, and scales as
$\bigotime(\deff^T \log T)$, playing a similar role as the $\bigotime(d\log T)$ bound
from LOCO.

\textbf{Approximate GD.} To trade off between computational complexity \big(smaller than
$\bigotime(d^2)$\big) and improved regret (close to $\bigotime(d\log T))$,
several methods try approximate second-order updates, replacing $\bH_t$
with an approximate~$\wt{\bH}_t$ that can be efficiently stored and inverted.
AdaGrad \cite{duchi_adaptive_2011} and ADAM \cite{kingma2014adam} reweight the gradient updates on a per-coordinate basis
using a diagonal $\wt{\bH}_t$, but these methods  ultimately only improve the regret dependency on $d$ and leave
the $\sqrt{T}$ component unchanged. Sketched-ONS, by \citet{luo_efficient_2016},
uses matrix sketching to approximate
$\bH_t$ with a $r$-rank sketch $\wt{\bH}_t$, that can be efficiently stored
and updated in $\bigotime(dr^2)$ time and space, close to the $\bigotime(d)$ complexity
of diagonal approximations.
More importantly, Sketched-ONS achieves a much smaller regret compared to diagonal
approximations: When the true $\bH_t$ is of low-rank~$r$, it recovers a
$\bigotime(r\log T)$ regret bound logarithmic in $T$. Unfortunately, due to the
sketch approximation, a new term appears in the bound that scales with the
spectra of $\bH_t$, and in some cases can grow much larger than
$\bigotime(\log T)$.

\textbf{Approximate kernel GD.}
Existing approximate GD methods for KOCO focus only on first-order updates,
trying to reduce the $\bigotime(t)$ per-step complexity. Budgeted methods,
such as Budgeted-GD \cite{wang2012breaking} and budgeted variants of the
perceptron \cite{cavallanti2007tracking, dekel2008forgetron, orabona2008projectron}
explicitly limit the size of the model, using some destructive budget maintenance
procedure (e.g., removal, projection) to constrain the natural model growth over time.
Alternatively, functional approximation methods in the primal \cite{lu_large_2015}
or dual \cite{le_dual_2016} use non-linear embedding techniques, such as
random feature expansion \cite{le2013fastfood}, to reduce
the KOCO problem to a LOCO problem and solve it efficiently.
Unfortunately, to guarantee $\bigotime(\sqrt{T})$ regret using
less than $\bigotime(t)$ space and time per round w.h.p., all of these methods require additional assumptions, such as points $\bx_t$ coming from a distribution
or strong convexity on the losses.
Moreover, as approximate first-order methods, they can at most
hope to match the $\bigotime(\sqrt{T})$ regret of exact GD,
and among second-order kernel methods, no approximation
scheme has been proposed that can provably maintain the same $\bigotime(\log T)$
regret as exact GD.
In addition, approximating $\bH_t$ is harder for KOCO, since 
we cannot directly access
the matrix representation of $\bH_t$ in the feature-space,
making diagonal approximation impossible, and low-rank sketching
harder.

\textbf{Contributions} In this paper, we introduce Kernel-ONS, an extension to KOCO of the
ONS algorithm. As a second-order method, \exactkons achieves a $\bigotime(\deff^t\log T)$
regret on a variety of curved losses, and runs in $\bigotime(t^2)$ time and space.
To alleviate the computational complexity, we propose \sketchkons, the first approximate second-order KOCO methods,
that approximates the kernel matrix with a low-rank sketch.
To compute this sketch we propose a new online kernel dictionary learning, \emph{kernel
online row sampling}, based on \emph{ridge leverage scores}. By adaptively increasing the size of its sketch,
\sketchkons provides a favorable regret-performance trade-off, where for
a given factor $\gamma \leq 1$, we can increase the regret by a linear $1/\gamma$ factor
to $\bigotime(\deff^t\log(T)/\gamma)$
while obtaining a quadratic~$\gamma^2$ improvement in runtime, thereby achieving $\bigotime(t^2\gamma^2)$
space and time cost per iteration.

\makeatletter{}\section{Background}\label{sec:setting}

In this section, we introduce linear algebra and RKHS notation,
and formally state the OCO problem in an RKHS~\citep{scholkopf2001learning}.

\textbf{Notation.}
We use upper-case bold letters $\:A$ for matrices,
lower-case bold letters $\:a$ for vectors,
lower-case letters $a$ for scalars.
We denote by $[\:A]_{ij}$ and $[\:a]_i$ 
the $(i,j)$ element of a matrix and $i$-th element of a vector respectively. We 
denote by $\bI_T \in \Real^{T \times T},$ the identity matrix of dimension~$T$ and 
by $\Diag(\:a)\in\Real^{T\times T}$, the diagonal matrix with the vector 
$\:a\in\Real^T$ on the diagonal. We use $\:e_{T,i} \in \Real^{T}$ to denote the 
indicator vector of dimension $T$ for element $i$. When the dimension of 
$\bI$ and $\:e_{i}$ is clear from the context, we omit the $T$.
We  also indicate with $\bI$ the identity operator. We use 
$\:A \succeq \:B$ to indicate that $\:A-\:B$ is a positive semi-definite (PSD) matrix.
With $\normsmall{\cdot}$ we indicate the \emph{operator} $\ell_2$-norm. Finally, the set 
of integers between 1 and $T$ is denoted by $[T] := \{1,\ldots,T\}$.

\textbf{Kernels.}
Given an arbitrary input space $\statespace$ and a positive definite kernel function
$\kerfunc: \statespace \times \statespace \rightarrow \Real$, we indicate the
\emph{reproducing kernel Hilbert space} (RKHS) associated with~$\kerfunc$ as
$\rkhs$.
We choose to represent our Hilbert space $\rkhs$ as a feature space
where, given $\kerfunc$,
we can find an associated feature map
$\varphi: \statespace \rightarrow \rkhs$, such that
$\kerfunc(\bx, \bx')$ can be expressed as an inner product $\kerfunc(\bx, \bx') = \hilbprod{\varphi(\bx)}{\varphi(\bx')}$.
With a slight abuse of notation, we represent our
feature space as an high-dimensional vector space, or in other words
$\rkhs \subseteq \Real^D$, where $D$ is very large or potentially infinite.
With this notation, we can write the inner product simply as
$\kerfunc(\bx, \bx') = \varphi(\bx)^\transp \varphi(\bx')$,
and for any function $\func_{\bw} \in \rkhs$, we can
represent it as a (potentially infinite) set of weights~$\bw$ such that
$\func_{\bw}(\bx) = \varphi(\bx)^\transp\bw$.
Given points $\{\bx_i\}_{i=1}^t$,
we shorten $\varphi(\bx_i) = \featvec_i$ and define the feature matrix $\featkermatrix_t = [\featvec_1, \dots, \featvec_t]\in\Re^{D\times t}$.
Finally, to denote the inner product between two arbitrary subsets $a$ and $b$ of
columns of $\featkermatrix_T$ we  use $\kermatrix_{a,b} = \featkermatrix_a^\transp\featkermatrix_b$.
With this notation, we can write the empirical kernel matrix as $\kermatrix_t = \kermatrix_{[t],[t]} = \featkermatrix_t^\transp\featkermatrix_t$, the vector with all the similarities between
a new point and the old ones as $\bk_{[t-1],t} = \featkermatrix_{t-1}^\transp\featvec_t$,
and the kernel evaluated at a specific point as $k_{t,t} = \featvec_t^\transp \featvec_t$.
Throughout the rest of the paper, we assume that $\kerfunc$ is normalized
and $\featvec_t^\transp\featvec_t = 1$.

\textbf{Kernelized online convex optimization.}
In the general OCO framework with linear prediction, the optimization process is a game
where at each time step $t \in [T]$ the player\vspace{-.7\baselineskip}
\begin{itemize}[leftmargin=*,itemsep=0pt]
    \item[1] receives an input $\bx_t \in \statespace$ from the adversary,
    \item[2] predicts $\wh{y}_t = f_{\bw_{t}}(\bx_t) = \varphi(\bx_t)^\transp\bw_{t} = \featvec_t^\transp\bw_{t}$,
    \item[3] incurs loss $\ell_t(\wh{y}_t)$, with $\ell_t$ a
    convex and differentiable function chosen by the adversary,
    \item[4] observes the derivative $\dot{g}_t = \ell'_t(\wh{y}_t)$.             \end{itemize}
\vspace{-.5\baselineskip}
Since the player uses a linear combination $\featvec_t^\transp\bw_{t}$
to compute $\wh{y}_t$, having observed $\dot{g}_t$, we can compute the gradient,
    \begin{align*}
    \bg_t = \nabla \ell_t(\wh{y}_t) = \dot{g}_t\nabla(\featvec_t^\transp\bw_{t-1}) = \dot{g}_t\featvec_t.
    \end{align*}
After $t$ timesteps, we indicate with $\dataset_t = \{\bx_i\}_{i=1}^{t}$,
the dataset containing the points observed so far. In the rest of the paper we consider the problem of \emph{kernelized OCO} (KOCO)
where $\rkhs$ is arbitrary and potentially non-parametric.
We refer to the special parametric case $\rkhs = \Real^d$ and $\featvec_t = \bx_t$
as \emph{linear OCO} (LOCO).

In OCO, the goal is to design an algorithm that returns a solution that performs
almost as well as the best-in-class, thus we must first define our comparison class.
We define the feasible set as $\feasibleset_t = \{\bw : |\featvec_t^\transp\bw| \leq C\}$
and $\feasibleset = \cap_{t=1}^T \feasibleset_t$.
This comparison class contains all functions $f_w$ whose output is contained
(clipped) in the interval $[-C,C]$ on all points $x_1,\ldots,x_T$. Unlike the often used constraint on $\normsmall{\bw}_{\rkhs}$ \cite{hazan_logarithmic_2006, zhu_online_2015},
comparing against clipped functions \cite{luo_efficient_2016, gammerman_2004, zhdanov_identity_2010} has a clear interpretation even when passing
from $\Real^d$ to $\rkhs$. Moreover,~$\feasibleset$~is invariant to linear
transformations of $\rkhs$ and suitable for practical problems where
it is often easier to choose a reasonable interval for the predictions $\wh{y}_t$ rather than
a bound on the norm of a (possibly non-interpretable) parametrization $\bw$.
We can now define the regret as
\begin{align*}
R_T(\bw) = \sum\nolimits_{t=1}^T \ell_t(\featvec_t^\transp\bw_t) - \ell_t(\featvec_t^\transp\bw)
\end{align*}
and denote with $R_T = R_T(\bw^{*})$, the regret w.r.t.\,\@$\bw^{*} = \argmin_{\bw \in \feasibleset} \sum\nolimits_{t=1}^T \ell_t(\featvec_t^\transp\bw)$,
i.e., the best fixed function in~$\feasibleset$.
We work with the following assumptions on the losses.
\begin{assumption}\label{ass:scalar-lipschitz}
 The loss function $\ell_t$ satisfies $|\ell'_t(y)| \leq L$ whenever $y \leq C$.
\end{assumption}
\vspace{-.5\baselineskip}
Note that this is equivalent to assuming Lipschitzness of the the loss w.r.t.\,$y$ and it is weaker than assuming something on the norm of the gradient $\normsmall{\bg_t}$,
since $\normsmall{\bg_t} = |\dot{g}_t|\normsmall{\featvec_t}$.
\begin{assumption}\label{ass:convexity}
There exists $\sigma_t \geq 0$ such that for all $\bu, \bw \in \feasibleset$ ,
$l_t(\bw) = \ell_t(\featvec_t^\transp\bw)$ is lower-bounded by
\begin{align*}
l_t(\bw) \geq
l_t(\bu) + \nabla l_t(\bu)^\transp(\bw \!-\! \bu)
+ \frac{\sigma_t}{2}(\nabla l_t(\bu)^\transp(\bw \!-\! \bu))^2.
 \end{align*}
\end{assumption}
This condition is weaker than strong convexity and it is satisfied by all
exp-concave losses~\cite{hazan_logarithmic_2006}. For example,
the squared loss $l_t(\bw)~=~(y_t~-~\bx_t^\transp\bw)^2$
is not strongly convex but satisfies Asm.\,\ref{ass:convexity} with $\sigma_t = 1/(8C^2)$ when $\bw \in \feasibleset$.

\makeatletter{}
\section{Kernelized Online Newton Step}\label{sec:exact-kern-ons}
\begin{algorithm}[t]
    \begin{algorithmic}[1]
        \renewcommand{\algorithmicrequire}{\textbf{Input:}}
        \renewcommand{\algorithmicensure}{\textbf{Output:}}
        \REQUIRE Feasible parameter $C$, stepsizes $\eta_t$, regulariz. $\alpha$
                \STATE Initialize $\bw_0 = \bsym{0},\bg_0 = \bsym{0},b_0 = 0$, $\bA_0 = \alpha\bI$
        \FOR{$t = \{1, \dots, T\}$}
            \STATE receive $\bx_t$
            \STATE compute $b_s$ as in Lem.~\ref{lem:exact-computable-reformulation}
            \STATE compute $\bu_t = \bA_{t-1}^{-1}(\sum_{s=0}^{t-1}b_s\bg_s)$
            \STATE compute $\wb{y}_t = \varphi(\bx_t)^\transp\bu_t$
            \STATE predict $\wh{y}_t = \varphi(\bx_t)^\transp\bw_t = \wb{y}_{t} - h(\wb{y}_t)$
            \STATE observe $\bg_t$, update $\bA_t = \bA_{t-1} + \eta_t \bg_t\bg_t^\transp$
        \ENDFOR
    \end{algorithmic}
    \caption{One-shot \exactkons}    \label{alg:ons-implicit}
\end{algorithm}

The online Newton step algorithm,
originally introduced by~\citet{hazan_logarithmic_2006},
is a projected gradient descent that uses the following
update rules
\begin{align*}
&\; \bu_t = \bw_{t-1} - \bA_{t-1}^{-1}\bg_{t-1},\\
&
\begin{aligned}
\bw_t = \Pi_{\feasibleset_t}^{\bA_{t-1}}(\bu_{t}),
\end{aligned}
\end{align*}
where $\Pi_{\feasibleset_t}^{\bA_{t-1}}(u_t) = \argmin_{\bw \in \feasibleset_t} \normsmall{\bu_{t} - \bw}_{\bA_{t-1}}$ is an oblique projection on a set $\feasibleset_t$ with matrix $\bA_{t-1}$. If $\feasibleset_t$ is the set of vectors with bounded prediction in $[-C,C]$ as by~\citet{luo_efficient_2016}, then the projection reduces to
\begin{align}\label{eq:feature.ons}
\bw_t = \Pi_{\feasibleset_t}^{\bA_{t-1}}(\bu_{t}) = \bu_{t} - \frac{h(\featvec_t^\transp\bu_t)}{\featvec_t^\transp\bA_{t-1}^{-1}\featvec_t}\bA_{t-1}^{-1}\featvec_t,
\end{align}
where $h(z) = \text{sign}(z)\max\{|z| - C,\;0\}$ computes how much $z$ is above
or below the interval $[-C,C]$. When $\bA_t = \bI/\eta_t$, ONS is equivalent to
vanilla projected gradient descent, which in LOCO
achieves $\bigotime(\sqrt{dT})$ regret \cite{zinkevich2003online}.
In the same setting,
\citet{hazan_logarithmic_2006} shows that
choosing $\bA_t = \sum_{s=1}^t \eta_s\bg_s\bg_s^\transp+ \alpha\bI$ makes
ONS an efficient reformulation of
\emph{follow the approximate leader} (FTAL).
While traditional follow-the-leader algorithms play the weight $\bw_t =
\argmin_{\bw \in \feasibleset_t} \sum_{s=1}^{t-1} l_t(\bw)$, FTAL
replaces the loss $l_t$ with a convex
approximation using Asm.\,\ref{ass:convexity}, and plays the minimizer of the
surrogate function.
As a result, under Asm.\,\ref{ass:scalar-lipschitz}-\ref{ass:convexity} and
when $\sigma_t \geq \sigma > 0$, FTAL achieves a logarithmic $\bigotime(d\log T)$ regret.
FTAL's solution path can be computed in $\bigotime(d^2)$ time using ONS
updates, and further speedups were proposed by \citet{luo_efficient_2016} using matrix sketching.

Unfortunately, in KOCO, vectors $\featvec_t$ and weights $\bw_t$
cannot be explicitly represented, and most of the quantities used in vanilla
ONS (Eq.\,\ref{eq:feature.ons}) cannot be directly computed. Instead,
we derive a closed form alternative (Alg.\,\ref{alg:ons-implicit}) that
can be computed in practice.
Using a rescaled variant of our feature vectors $\featvec_t$, $\wb{\featvec}_t~=~\dot{g}_t\sqrt{\eta_t}\featvec_t = \sqrt{\eta_t}\bg_t$ and $\wb{\featkermatrix}_t = [\wb{\featvec}_1, \dots, \wb{\featvec}_t]$,
we can rewrite $\bA_t = \wb{\featkermatrix}_t\wb{\featkermatrix}_t^\transp + \alpha\bI$
and $\wb{\featkermatrix}_t^\transp\wb{\featkermatrix}_t = \wb{\kermatrix}_t$,
where the empirical kernel matrix~$\wb{\kermatrix}_t$ is computed using the
rescaled kernel
$\wb{\kerfunc}(\bx_i, \bx_j) = \dot{g}_i\sqrt{\eta_i}\dot{g}_j\sqrt{\eta_j}\kerfunc(\bx_i, \bx_j)$
instead of the original $\kerfunc$, or equivalently
$\wb{\kermatrix}_t = \bD_t\kermatrix_t\bD_t$ with
$\bD_t = \Diag(\{\dot{g}_i\sqrt{\eta_i}\}_{i=1}^t)$ the rescaling diagonal matrix.
We begin by noting that
\begin{align*}
\wh{y}_t &= \featvec_t^\transp\bw_t = \featvec_t^\transp\left(\bu_t - \frac{h(\featvec_t^\transp\bu_t)}{\featvec_t^\transp\bA_{t-1}^{-1}\featvec_t}\bA_{t-1}^{-1}\featvec_t\right)\\
&= \featvec_t^\transp\bu_t - h(\featvec_t^\transp\bu_t)\frac{\featvec_t^\transp\bA_{t-1}^{-1}\featvec_t}{\featvec_t^\transp\bA_{t-1}^{-1}\featvec_t}
= \wb{y}_t - h(\wb{y}_t).
\end{align*}
As a consequence, if we can find a way to compute $\wb{y}_t$, then we can obtain~$\wh{y}_t$
without explicitly computing $\bw_t$. Before that, we first derive a non-recursive formulation of $\bu_t$.
\begin{lemma}\label{lem:exact-kern-ons-u-reform}
In Alg.\,\ref{alg:ons-implicit} we introduce
\begin{align*}
 b_i = [\bb_{t}]_{i} &= \dot{g}_i\sqrt{\eta_i}\left(\wh{y}_i - \frac{h(\wb{y}_i)}{\wb{\featvec}_i^\transp\bA_{i-1}^{-1}\wb{\featvec}_i}\right)- \frac{1}{\sqrt{\eta_i}}
\end{align*}
and compute $\bu_t$ as
\begin{align*}
\bu_{t} &= \bA_{t-1}^{-1}\wb{\featkermatrix}_{t-1}\bb_{t-1}.
\end{align*}
Then,	 $\bu_t$ is equal to the same quantity in Eq.\,\ref{eq:feature.ons} and the sequence of predictions $\wh y_t$ is the same in both algorithms.
\end{lemma}

While the definition of $\bb_t$ and $\bu_t$ still requires performing operations in the (possibly infinitely dimensional) feature space, in the following we show that~$\bb_t$ and the prediction~$\wb{y}_t$ can be conveniently computed using only inner products.
\begin{lemma}\label{lem:exact-computable-reformulation}
All the components $b_i = [\bb_t]_{i}$ of the vector introduced in Lem.\,\ref{lem:exact-kern-ons-u-reform} can be computed as
\begin{align*}
    \dot{g}_i\sqrt{\eta_i}\bigg(\wh{y}_i 
    - \frac{\alpha h(\wb{y}_i)}{k_{i,i} - \wb{\bk}_{[i-1],i}^\transp(\wb{\kermatrix}_{i-1} + \alpha\bI)^{-1}\wb{\bk}_{[i-1],i}}- \frac{1}{\eta_i}\bigg).
    \end{align*}
Then, we can compute
\begin{align*}
    &\wb{y}_t = \frac{1}{\alpha}\bk_{[t-1],t}^\transp\bD_{t-1}(\bb_{t-1} - (\wb{\kermatrix}_{t-1} + \alpha\bI)^{-1}\wb{\kermatrix}_{t-1}\bb_{t-1}).
\end{align*}
\end{lemma}
Since Alg.\,\ref{alg:ons-implicit} is equivalent to ONS (Eq.\,\ref{eq:feature.ons}), 
existing regret bounds for ONS directly applies to its kernelized version.

\begin{proposition}[\citealp{luo_efficient_2016}]\label{prop:exact-ons-regret}
For any sequence of losses $\ell_t$ satisfying
Asm.\,\ref{ass:scalar-lipschitz}-\ref{ass:convexity}, the regret $R_T$ of
Alg.\,\ref{alg:ons-implicit} is bounded by $R_T \leq \alpha\normsmall{\bw^*}^2 + R_G + R_D$ 
with
\begin{align*}
R_G :=& \sum_{t=1}^T \bg_t^\transp\bA_t^{-1}\bg_t = \sum_{t=1}^T \wb{\featvec}_t^\transp(\wb{\featkermatrix_t}\wb{\featkermatrix}_t^\transp + \alpha\bI)^{-1}\wb{\featvec}_t/\eta_t\\
R_D :=& \sum_{t=1}^T (\bw_t - \bw^*)^\transp(\bA_t - \bA_{t-1} \!-\! \sigma_t\bg_t\bg_t^\transp)(\bw_t - \bw^*)\\
=&\sum_{t=1}^T (\eta_t - \sigma_t)\dot{g}_t^2(\featvec_t^\transp(\bw_t - \bw^*))^2.
\end{align*}
\end{proposition}
In the $d$-dimensional LOCO, choosing a decreasing step-size $\eta_t = \sqrt{d/(C^2L^2t)}$
allows ONS to achieve a $\bigotime(CL\sqrt{dT})$ regret for the cases where $\sigma_t = 0$.
When $\sigma_t \geq \sigma > 0$ (e.g., when the functions are exp-concave) we can set
$\eta_t = \sigma_t$ and improve the regret to $\bigotime(d\log(T))$.
Unfortunately, these quantities hold little meaning for KOCO  with $D$-dimensional features,
since a $\bigotime(\sqrt{D})$ regret can be very large or even infinite. On the other hand, we expect the regret of KONS to depend on quantities that are more strictly related to the kernel $\wb{\kermatrix}_t$ and its complexity.
\begin{definition}
Given a kernel function $\kerfunc$, a set of points $\dataset_t = \{\bx_i\}_{i=1}^t$
and a parameter $\alpha > 0$, we define the $\alpha$-ridge leverage scores
(RLS) of point $i$ as
\begin{align}\label{eq:exact-rls}
\tau_{t,i} \!=\! \be_{t,i}^\transp\kermatrix_{t}^\transp (\kermatrix_t \!+\! \alpha\bI)^{\text{--}1} \:e_{t,i} \!=\! \featvec_i^\transp(\featkermatrix_t\featkermatrix_t^\transp \!+\! \alpha\bI)^{\text{--}1}\featvec_i,
\end{align}
and the effective dimension of $\dataset_t$ as
\begin{align}\label{eq:exact-deff}
        \deff^{t}(\alpha) &= \sum_{i=1}^t \tau_{t,i} = \Tr\left(\kermatrix_t(\kermatrix_t + \alpha \bI_t)^{-1}\right).
\end{align}
\end{definition}
In general, leverage scores have been used to measure the correlation between a point $i$ w.r.t.\,the other $t-1$ points, and therefore how essential it is in
characterizing the dataset \cite{alaoui2014fast}. As an example, if $\featvec_i$ is completely orthogonal
to the other points,
$\tau_{t,i} = \featvec_i^\transp(\featvec_i\featvec_i^\transp + \alpha\bI)^{-1}\featvec_i \leq 1/(1+\alpha)$
and its RLS is maximized,
while in the case where all the points $\bx_i$ are identical,
$\tau_{t,i} = \featvec_i^\transp(t\featvec_i\featvec_i^\transp + \alpha\bI)^{-1}\featvec_i \leq 1/(t+\alpha)$
and its RLS is minimal. While the previous definition is provided for a generic kernel function $\kerfunc$, we can easily instantiate it on $\wb\kerfunc$ and obtain the definition of $\wb\tau_{t,i}$. 
By recalling the first regret term in the decomposition of Prop.\,\ref{prop:exact-ons-regret}, we notice that
\[
R_G = \sum_{t=1}^T \wb{\featvec}_t^\transp(\wb{\featkermatrix_t}\wb{\featkermatrix}_t^\transp + \alpha\bI)^{-1}\wb{\featvec}_t/\eta_t
= \sum_{t=1}^T \wb{\tau}_{t,t}/\eta_t,
\]
which reveals a deep connection between the regret of KONS and the cumulative sum of the RLS. In other words, the RLS capture how much
the adversary can increase the regret by picking orthogonal directions that have not been seen before. While in  LOCO,  this can happen at most~$d$~times (hence the dependency on $d$ in the final regret, which is mitigated by a suitable choice of $\eta_t$), in  KOCO,  $R_G$ can grow linearly with time,
since large $\rkhs$ can have infinite near-orthogonal directions. Nonetheless, the actual growth rate is now directly related to the complexity of the sequence of points chosen by the adversary and the kernel function $\kerfunc$. While the effective dimension $\deff^t(\alpha)$ is related to the capacity of the RKHS $\mathcal{H}$ on the points in~$\dataset_t$ and it has been shown to characterize the generalization error in batch linear regression~\cite{rudi2015less}, we see that $R_G$ is rather related to the \textit{online} effective dimension $\bdonl^t(\alpha) = \sum_i \wb{\tau}_{i,i}$. Nonetheless, we show that the two quantities are also strictly related to each other.

\begin{lemma}\label{lem:bound-online-tau}
For any dataset $\dataset_T$, any $\alpha > 0$ we have
\begin{align*}
\bdonl^T(\alpha) := \sum_{t=1}^{T} \wb{\tau}_{t,t} &\leq \log(\Det(\wb{\kermatrix}_T/\alpha + \bI))\\
&\leq \bdeff^T(\alpha) (1 + \log(\normsmall{\wb{\kermatrix}_T}/\alpha + 1)).
\end{align*}
\end{lemma}

We first notice that in the first inequality we relate $\bdonl^T(\alpha)$ to the log-determinant of the kernel matrix $\wb{\kermatrix}_T$. This quantity appears in a large number of works on online linear prediction \citep{cesa2005second,srinivas2010gaussian} where they were connected to the maximal mutual information gain in Gaussian processes. Finally, the second inequality shows that in general the complexity of online learning is only a factor $\log T$ (in the worst case) away from the complexity of batch learning. At this point, we can generalize the regret bounds of LOCO to KOCO.

\begin{theorem}\label{thm:kernel-ons-regret}
For any sequence of losses $\ell_t$ satisfying
Asm.\,\ref{ass:scalar-lipschitz}-\ref{ass:convexity}, let
$\sigma = \min_t \sigma_t$. If $\eta_t \geq \sigma \geq 0$ for all $t$ and
$\alpha \leq \sqrt{T}$, the regret of
Alg.\,\ref{alg:ons-implicit}
is upper-bounded as
\begin{align*}
R_T \leq \alpha\normsmall{\bw^{*}}^2 + \donl^T(\alpha)/\eta_T + 4C^2 L^2\sum_{t=1}^T (\eta_t - \sigma).
\end{align*}
In particular, if for all $t$ we have $\sigma_t \geq \sigma > 0$, setting $\eta_t = \sigma$
we obtain
\begin{align*}
R_T \leq \alpha\normsmall{\bw^*}^2 + 2\deff^T\left(\alpha/(\sigma L^2)\right)\frac{\log(2 \sigma L^2 T)}{\sigma}\CommaBin
\end{align*}
otherwise, $\sigma = 0$ and setting $\eta_t = 1/(LC\sqrt{t})$ we obtain
\begin{align*}
R_T \leq \alpha\normsmall{\bw^{*}}^2 + 4LC\sqrt{T}\deff^T(\alpha/L^2)\log(2L^2T).
\end{align*}
\end{theorem}

\textbf{Comparison to LOCO algorithms.} 
We first notice that the effective dimension $\deff^T(\alpha)$ can be seen as a soft rank for $\wb{\kermatrix}_T$ and that it is smaller than the rank $r$ for any $\alpha$.\footnote{This can be easily seen as $\deff^T(\alpha) = \sum_t \lambda_t / (\lambda_t+\alpha)$, where $\lambda_t$ are the eigenvalues of $\wb{\kermatrix}_T$.} For exp-concave functions (i.e., $\sigma > 0$), we slightly improve over
the bound of~\citet{luo_efficient_2016} from $\bigotime(d\log T)$ down to
$\bigotime(\deff^T(\alpha)\log T) \leq \bigotime(r\log T)$, where $r$ is the
(unknown) rank of the dataset.
Furthermore, when $\sigma\!=\! 0$, setting $\eta_t \!=\! \sqrt{1/(L^2C^2 t)}$ gives us
a regret $\bigotime(\sqrt{T}\deff^T(\alpha)) \!\leq\!
\bigotime(\sqrt{T}r)$, which is potentially much smaller than $\bigotime(\sqrt{Td})$. Furthermore, if an oracle
provided us in advance with $\deff^T(\alpha)$, setting
$\eta_t~=~\sqrt{\deff^T(\alpha)/(L^2C^2 t)}$
gives a regret $\bigotime(\sqrt{\deff^T(\alpha)T}) \leq \bigotime(\sqrt{rT})$.

\textbf{Comparison to KOCO algorithms.} Simple functional gradient descent
(e.g.,\,\textsc{NORMA}, \citealp{kivinen_online_2004}) achieves a
$\bigotime(\sqrt{T})$ regret when properly tuned \cite{zhu_online_2015},
regardless of the loss function.
For the special case of squared loss, \citet{zhdanov_identity_2010} show that Kernel Ridge Regression
achieves the same $\bigotime(\log(\Det(\wb{\kermatrix}_T/\alpha + \bI)))$ regret
as achieved by \exactkons for general exp-concave losses.
 
\makeatletter{}
\section{Kernel Online Row Sampling}\label{sec:online-squeak}

\begin{algorithm}[t]
    \begin{algorithmic}[1]
        \renewcommand{\algorithmicrequire}{\textbf{Input:}}
        \renewcommand{\algorithmicensure}{\textbf{Output:}}
        \REQUIRE Regularization $\alpha$, accuracy $\varepsilon$, budget $\beta$
                \STATE Initialize $\coldict_0 = \emptyset$
        \FOR{$t = \{0, \dots, T-1\}$}
            \STATE receive $\wb{\featvec}_t$
            \STATE construct temporary dictionary $\wb{\coldict}_t := \coldict_{t-1} \cup (t,1)$
            \STATE compute $\wt{p}_{t} = \min\{\beta\atau_{t,t}, 1\}$ using $\wb{\coldict}_t$ and Eq.~\ref{eq:rls-est-online}
                        \STATE draw $z_{t} \sim \mathcal{B}(\wt{p}_{t})$ and if $z_{t} = 1$, add $(t, 1/\wt{p}_{t})$ to $\coldict_t$
                    \ENDFOR
    \end{algorithmic}
    \caption{Kernel Online Row Sampling (\onlinesqueak)}    \label{alg:onsqueak}
\end{algorithm}

Although \exactkons achieves a low regret, storing and inverting the $\bkermatrix$
matrix requires $\bigotime(t^2)$ space and $\bigotime(t^3)$ time, which becomes
quickly unfeasible as $t$ grows.
To improve space and time efficiency, we replace $\bkermatrix_t$ with an
accurate low-rank approximation $\akermatrix_t$, constructed using a carefully
chosen dictionary $\coldict_t$ of points from $\dataset_t$.
We extend the \emph{online row sampling} (ORS)
algorithm of~\citet{cohen2016online} to the kernel setting and obtain Kernel-ORS (Alg.\,\ref{alg:onsqueak}). There are two main
obstacles to overcome in the adaptation of ORS: From an algorithmic perspective we need
to find a computable estimator for the RLS, since $\featvec_t$ cannot be
accessed directly, while from an analysis perspective we must prove that our
space and time complexity does not scale with the dimension of
$\featvec_t$ (as~\citealt{cohen2016online}), as it can potentially be infinite.

We define a dictionary $\coldict_t$ as a collection of \textit{(index, weight)} tuples $(i,1/\wt{p}_i)$
and the associated selection matrix $\selmatrix_t \in \Real^{t \times t}$ as a diagonal matrix with $1/\sqrt{\wt{p}_{i}}$ for all $i \in \coldict_t$ and 0 elsewhere.
We also introduce $\bA_t^{\coldict_t} = \wb{\featkermatrix}_{t}\selmatrix_{t}\selmatrix_{t}^\transp\wb{\featkermatrix}_{t}^\transp + \alpha\bI$ as an approximation of $\bA_t$ constructed using the dictionary $\coldict_t$.
At each time step, \onlinesqueak temporarily adds $t$ with weight 1 to the dictionary
$\coldict_{t-1}$ and constructs the temporary dictionary $\coldict_{t,*}$ and the corresponding selection matrix $\selmatrix_{t,*}$ and approximation $\bA_{t}^{\coldict_{t,*}}$. This augmented dictionary can be effectively used to compute the RLS estimator,
\begin{align}\label{eq:rls-est-online}
&\atau_{t,i} = (1 + \epsilon)\wb{\featvec}_t\big(\bA_{t}^{\coldict_{t,*}}\big)^{-1}\wb{\featvec}_t\\
&\!=\tfrac{1 + \epsilon}{\alpha}\big(\wb{k}_{t,t} - \wb{\bk}_{[t],t}^\transp\selmatrix_{t,*}(\selmatrix_{t,*}^\transp\wb{\kermatrix}_t\selmatrix_{t,*} + \alpha\bI)^{-1}\selmatrix_{t,*}^\transp\wb{\bk}_{[t],t}\big).\nonumber
\end{align}
While we introduced a similar estimator before~\citep{calandriello_disqueak_2017}, here we modified it so that $\atau_{t,i}$ is an overestimate of
the actual $\wb{\tau}_{t,i}$.
Note that all rows and columns for which $\selmatrix_{t,*}$ is zero (all points
outside the temporary dictionary $\coldict_{t,*}$) do not influence the estimator, so they can be excluded
from the computation. As a consequence, denoting by $|\coldict_{t,*}|$ the
size of the dictionary, $\atau_{t,i}$ can be efficiently computed in $\bigotime(|\coldict_{t,*}|^2)$ space
and $\bigotime(|\coldict_{t,*}|^2)$ time (using an incremental update of Eq.\,\ref{eq:rls-est-online}).
After computing the RLS, \onlinesqueak randomly chooses whether to include a point in the
dictionary using a coin-flip with probability $\wt{p}_t = \min\{\beta\atau_{t,t},1\}$ and weight $1/\wt{p}_t$, where $\beta$ is a parameter.
The following theorem gives us at each step guarantees on the accuracy of the
approximate matrices $\bA_t^{\coldict_t}$ and of estimates $\atau_{t,t}$,
as well as on the size $|\coldict_t|$ of the dictionary.
\begin{theorem}\label{thm:online-squeak}
Given parameters $0< \varepsilon \leq 1$, $0< \alpha$, $0 <\delta <1$,
let $\rho = \frac{1+\varepsilon}{1-\varepsilon}$ and
run Algorithm~\ref{alg:onsqueak} with 
$\beta \geq 3\log(T/\delta)/\varepsilon^2$. Then w.p.\,$1 - \delta$,
for all steps $t \in [T]$,
\begin{itemize}[itemsep=-2.5pt, ]
\item[\textbf{(1)}]  $ (1-\varepsilon)\bA_t \preceq  \bA_t^{\coldict_t} \preceq (1+\varepsilon)\bA_t$.
\item[\textbf{(2)}] The dictionary's size $|\coldict_t| = \sum_{s=1}^t z_{s} $ is bounded by
\begin{align*}
\hspace{-.5cm}
\sum_{s=1}^t z_{s} \leq 3\sum_{s=1}^t \wt{p}_{s}
\leq \donl^t(\alpha)\frac{3\rho\beta}{\varepsilon^2}
&\leq \deff^t(\alpha) \frac{6\rho\log^2\left(\frac{2T}{\delta}\right)}{\varepsilon^{2}}.
    \end{align*}
\item[\textbf{(3)}] Satisfies $\tau_{t,t} \leq \atau_{t,t} \leq \rho\tau_{t,t}$.
\end{itemize}
Moreover, the algorithm runs in $\bigotime(\deff^t(\alpha)^2\log^4(T))$
space, and $\wt{\bigotime}(\deff^t(\alpha)^2)$ time per
iteration.
\end{theorem}

The most interesting aspect of this result is that the dictionary $\coldict_t$ generated by \onlinesqueak allows to accurately approximate the $\bA_t = \wb{\featkermatrix}_{t}\wb{\featkermatrix}_{t}^\transp + \alpha\bI$ matrix up to a small $(1 \pm \varepsilon)$ multiplicative factor with a small time and space complexity, which makes it a natural candidate to sketch \exactkons.
 
\makeatletter{}\section{Sketched ONS}\label{sec:skethed-ons}

\begin{algorithm}[t]
    \begin{algorithmic}[1]
        \renewcommand{\algorithmicrequire}{\textbf{Input:}}
        \renewcommand{\algorithmicensure}{\textbf{Output:}}
        \REQUIRE Feasible parameter $C$, stepsizes $\eta_t$, regulariz. $\alpha$
                \STATE Initialize $\bw_0 = \bsym{0},\bg_0 = \bsym{0},b_0 = 0$, $\wt{\bA}_0 = \alpha\bI$
        \STATE Initialize independent run of \onlinesqueak
        \FOR{$t = \{1, \dots, T\}$}
            \STATE receive $\bx_t$
            \STATE compute $\wt{\bu}_t = \wt{\bA}_{t-1}^{-1}(\sum_{s=0}^{t-1}\wt{b}_s\bg_s)$
            \STATE compute $\breve{y}_t = \varphi(\bx_t)^\transp\wt{\bu}_t$
            \STATE predict $\wt{y}_t = \varphi(\bx_t)^\transp\wt{\bw}_t = \breve{y}_{t} - h(\breve{y}_t)$, observe $\bg_t$
            \STATE compute $\atau_{t,t}$ using \onlinesqueak (Eq.\,\ref{eq:rls-est-online})
            \STATE compute $\wt{p}_{t} = \max\{\min\{\beta\atau_{t,t}, 1\},\gamma\}$
            \STATE draw $z_{t} \sim \mathcal{B}(\wt{p}_{t})$
            \STATE update $\wt{\bA}_t = \wt{\bA}_{t-1} + \eta_t z_t \bg_t\bg_t^\transp$
        \ENDFOR
    \end{algorithmic}
    \caption{\sketchkons}    \label{alg:ons-sketch}
\end{algorithm}

Building on \onlinesqueak, we now introduce a sketched variant of KONS that can efficiently trade off
between computational performance and regret.
Alg.\,\ref{alg:ons-sketch} runs \onlinesqueak as a black-box estimating RLS $\atau_t$, that are then used to sketch the original matrix $\bA_t$
with a matrix $\wt{\bA}_t = \sum_{s=1}^t \eta_t z_t \bg_t\bg_t^\transp$,
where at each step we add the current gradient $\bg_t\bg_t^\transp$ only
if the coin flip $z_t$ succeeded. Unlike \onlinesqueak,
the elements added to $\wt{\bA}_t$ are not weighted, and the probabilities $\wt{p}_{t}$
used for the coins $z_t$ are chosen as the maximum between $\atau_{t,t}$,
and a parameter $0 \leq \gamma \leq 1$.
Let $\bR_t$ be the unweighted counterpart of $\bS_t$, that is
$[\bR_t]_{i,j} = 0$ if $[\bS_t]_{i,j} = 0$ and $[\bR_t]_{i,j} = 1$
if $[\bS_t]_{i,j} \neq 0$. Then we can efficiently compute the coefficients
$\wt{b}_t$ and predictions $\wt{y}_t$ as follows.
\begin{lemma}\label{lem:sketched-computable-reformulation}
Let $\bE_t = \selrmatrix_{t}^\transp\wb{\kermatrix}_{t}\selrmatrix_{t} + \alpha\bI$ be an auxiliary matrix, then
all the components $\wt{b}_i = [\wt{\bb}_t]_{i}$ used in Alg.\,\ref{alg:ons-sketch} can be computed as
\begin{align*}
    \dot{g}_i\sqrt{\eta_i}\bigg(\wt{y}_i 
    - \frac{\alpha h(\breve{y}_i)}{k_{i,i} - \wb{\bk}_{[i-1],i}^\transp\selrmatrix_{i-1}\bE_{i-1}^{-1}\selrmatrix_{i-1}\wb{\bk}_{[i-1],i}} - \frac{1}{\eta_i}\bigg).
\end{align*}
Then we can compute
\begin{align*}
    \breve{y}_t = \frac{1}{\alpha}\big(&\bk_{[t-1],t}^\transp\bD_{t-1}\bb_{t-1}\\
    &- \bk_{[t-1],t}^\transp\bD_{t-1}\selrmatrix_{t-1}\bE_{t-1}^{-1}\selrmatrix_{t-1}\wb{\kermatrix}_{t-1}\bb_{t-1}\big).
\end{align*}
\end{lemma}

Note that since the columns in $\bR_t$ are selected without weights,
$(\selrmatrix_{t}^\transp\wb{\kermatrix}_{t}\selrmatrix_{t} + \alpha\bI)^{-1}$
can be updated efficiently using block inverse updates, and only when
$\wt{\bA}_t$ changes. While the specific reason for choosing the unweighted sketch $\wt{\bA}_t$ instead of the weighted version $\bA_t^{\coldict_t}$ used in \onlinesqueak is discussed further in Sect.~\ref{sec:discussion}, the following corollary shows that $\wt{\bA}_t$ is as accurate as $\bA_t^{\coldict_t}$ in approximating $\bA_t$ up to the smallest sampling probability $\wt{p}_t^\gamma$.

\begin{corollary}\label{cor:skons-from-kors}
Let $\wt{p}_{\min}^\gamma = \min_{t=1}^T \wt{p}_t^\gamma$. Then w.h.p., we have
\begin{align*}
(1-\varepsilon)\wt{p}_{\min}\bA_t \preceq \wt{p}_{\min}\bA_t^{\coldict_t} \preceq \wt{\bA}_t.
\end{align*}
\end{corollary}

We can now state the main result of this section.
Since for \sketchkons we are interested not only in regret minimization, but also
in space and time complexity, we do not consider the case $\sigma = 0$, because
when the function does not have any curvature, standard GD already achieves
the optimal regret of $\bigotime(\sqrt{T})$ \cite{zhu_online_2015} while requiring only $\bigotime(t)$
space and time per iteration.

\begin{theorem}\label{thm:sketch-ons-main-thm}
For any sequence of losses $\ell_t$ satisfying
Asm.\,\ref{ass:scalar-lipschitz}-\ref{ass:convexity}, let
$\sigma = \min_t \sigma_t$ and $\wb{\tau}_{\min} = \min_{t=1}^T \wb{\tau}_{t,t}$.
When $\eta_t \geq \sigma > 0$ for all $t$,
$\alpha \leq \sqrt{T}$, $\beta \geq 3\log(T/\delta)/\varepsilon^2$,
if we set $\eta_t = \sigma$ then w.p.\,$1-\delta$ the regret of
Alg.\,\ref{alg:ons-sketch}
satisfies
\begin{align}\label{eq:skons-regret}
\wt{R}_T \leq \alpha\normsmall{\bw^*}^2 + 2\frac{\deff^T\left(\alpha/(\sigma L^2)\right)\log(2 \sigma L^2 T)}{\sigma\max\{\gamma,\beta\wb{\tau}_{\min}\}}\CommaBin
\end{align}
and the algorithm runs in $\bigotime(\deff^t(\alpha)^2 + t^2\gamma^2)$ time and $\bigotime(\deff^t(\alpha)^2 + t^2\gamma^2)$ space
complexity for each iteration $t$.

\end{theorem}

\textbf{Proof sketch:} Given these guarantees, we need to bound
$R_G$ and $R_D$. Bounding $R_D$ is straightforward, since by construction
\sketchkons adds at most $\eta_t\bg_t\bg_t^\transp$ to~$\wt{\bA}_t$ at each step.
To bound $R_G$ instead, we must take into account that an unweighted
$\wt{\bA}_t = \wb{\featkermatrix}_t\selrmatrix_t\selrmatrix_t^\transp\wb{\featkermatrix}_t^\transp + \alpha\bI$
can be up to $\wt{p}_{\min}$ distant from the weighted $\wb{\featkermatrix}_t\selmatrix_t\selmatrix_t^\transp\wb{\featkermatrix}_t^\transp$ for which we have guarantees. Hence the $\max\{\gamma,\beta\wb{\tau}_{\min}\}$ term
appearing at the denominator.

\section{Discussion}\label{sec:discussion}

\textbf{Regret guarantees.} From Eq.\,\ref{eq:skons-regret} we can see that
when $\wb{\tau}_{\min}$ is not too small, setting $\gamma = 0$ we recover the
guarantees of exact \exactkons. Since usually we do not know $\wb{\tau}_{\min}$, we can choose to set $\gamma > 0$, and as long as
$\gamma \geq 1/\polylog T$, we preserve a (poly)-logarithmic regret.

\textbf{Computational speedup.}
The time required to compute $\bk_{[t-1],t}$,
$k_{t,t}$, and $\bk_{[t-1],t}^\transp\bD_{t-1}\bb_{t-1}$ gives a minimum $\bigotime(t)$
per-step complexity. Note that $\wb{\kermatrix}_{t-1}\bb_{t-1}$ can also be computed
incrementally in $\bigotime(t)$ time.
Denoting the size of the dictionary at time $t$ as $B_t = \wt{\bigotime}(\deff(\alpha)_t + t\gamma)$,
computing $[\wt{\bb}_t]_{i}$ and $\bk_{[t-1],t}^\transp\bD_{t-1}\selrmatrix_{t-1}\bE_{t-1}^{-1}\selrmatrix_{t-1}\wb{\kermatrix}_{t-1}\bb_{t-1}$ requires an additional $\bigotime(B_t^2)$ time.
When $\gamma \leq \deff^t(\alpha)/t$, each iteration takes $\bigotime(\deff^t(\alpha)^2)$
to compute $\atau_{t,t}$ incrementally using \onlinesqueak,
$\bigotime(\deff^t(\alpha)^2)$ time to update $\wt{\bA}_t^{-1}$
and $\bigotime(\deff^t(\alpha)^2)$ time to compute $[\bb_t]_t$.
When $\gamma > \deff^t(\alpha)/t$, each iteration still takes $\bigotime(\deff^t(\alpha)^2)$
to compute $\atau_{t,t}$ using \onlinesqueak and
$\bigotime(t^2\gamma^2)$ time to update the inverse
and compute~$[\bb_t]_t$.
Therefore, in the case when $\wb{\tau}_{\min}$ is not too small,
our runtime is of the order $\bigotime(\deff^t( \alpha )^2 + t)$, which is
almost as small as the $\bigotime(t)$ runtime of GD but with the advantage of a second-order method logarithmic regret. 
Moreover, when
$\wb{\tau}_{\min}$ is small and we set a large $\gamma$, we can trade off
a $1/\gamma$ increase in regret for a $\gamma^2$ decrease in space and
time complexity when compared to exact \exactkons (e.g., setting
$\gamma = 1/10$ would correspond to a tenfold increase in regret, but a hundred-fold
reduction in computational complexity).

\textbf{Asymptotic behavior.} Notice however, that space and time complexity,
grow roughly with a term $\Omega(t\min_{s=1}^t \wt{p}_{s}) \sim \Omega(t\max\{\gamma,\beta\wb{\tau}_{\min}\})$,
so if this quantity does not
decrease over time, the computational cost of \sketchkons will remain large and
close to exact \exactkons. This is to be expected, since \sketchkons must always keep an accurate sketch in order to guarantee a logarithmic regret bound.
Note that \citet{luo_efficient_2016} took an opposite approach for LOCO, where they keep a
fixed-size sketch but possibly pay in regret, if this fixed size happens to be too small.
Since a non-logarithmic regret is achievable simply running vanilla GD, we rather opted for an adaptive sketch at the cost of space and time complexity.
In batch optimization,
where $\ell_t$ does not change over time, another possibility is to stop updating the solution once $\wb{\tau}_{\min}$
becomes too small. When $\bH_s$ is the Hessian of $\ell$ in $\bw_s$,
then the quantity $\bg_t^\transp\bH_t^{-1}\bg_t$, in the context of
Newton's method, is called \emph{Newton decrement} and it 
corresponds up to constant factors to $\wb{\tau}_{\min}$. Since a stopping condition based on Newton's decrement is directly related to the near-optimality of the current $\bw_t$ \cite{nesterov1994interior}, stopping when $\wb{\tau}_{\min}$ is small also provides guarantees about the quality of the solution.

\textbf{Sampling distribution.} Note that although $\gamma > 0$ means
that all columns have a small uniform chance of being selected for
inclusion in $\wt{\bA}_t$, this is \emph{not} equivalent to uniformly sampling
columns. It is rather a combination of a RLS-based sampling
to ensure that columns important to reconstruct
$\bA_t$ are selected and a threshold on the probabilities to avoid too much variance in the estimator.

\textbf{Biased estimator and results in expectation.} The random approximation $\wt{\bA}_t$ is biased, since
$\expectedvalue[\wb{\featkermatrix}_t\selrmatrix_t\selrmatrix_t^\transp\wb{\featkermatrix}_t^\transp] = \wb{\featkermatrix}_t\Diag(\{\wb{\tau}_{t,t}\})\wb{\featkermatrix}_t^\transp \neq \wb{\featkermatrix}_t\wb{\featkermatrix}_t^\transp$.
Another option would be to use a weighted and unbiased 
$\wt{\bA}'_t = \sum_{s=1}^t \eta_s z_s/\wt{p}_s\bg_s\bg_s^\transp$ approximation, used in \onlinesqueak and a common choice in matrix approximation methods, see e.g.,~\citealp{alaoui2014fast}.
Due to its unbiasedness, this variant would automatically
achieve the same logarithmic regret as exact \exactkons \textit{in expectation} (similar to the result obtained by~\citealp{luo_efficient_2016},
using Gaussian random projection in LOCO).
While any unbiased estimator, e.g., uniform sampling of~$\bg_t$, would achieve
this result, RLS-based sampling already provides strong reconstruction guarantees
sufficient to bound $R_G$. Nonetheless, the weights $1/\wt{p}_s$ may cause large variations in $\wt{\bA}_t$ over consecutive steps, thus leading to a large regret $R_D$ in high probability.

\textbf{Limitations of dictionary learning approaches and open problems.}
From the discussion above, it appears that a weighted, unbiased dictionary may not achieve high-probability logarithmic guarantee
 because of the high variance coming 
 from sampling. On the other hand, if we want to recover the regret guarantee, we may have to pay for it with a large dictionary. 
 This may actually be due to the analysis, the algorithm, or the setting. 
An important property of the dictionary learning approach used in \onlinesqueak is that it can only add \emph{but not remove} columns and potentially re-weight them.
 Notice that in the batch setting \cite{alaoui2014fast,calandriello_disqueak_2017}, the sampling of columns does not cause any issue
 and we can have strong learning guarantees in high probability with a small dictionary.
Alternative sketching methods such as Frequent Directions \citep[FD,][]{ghashami2016frequent} do \emph{create new atoms} as learning progresses.
By restricting to composing dictionaries from existing columns, 
we only have the degree of freedom of the weights of the columns. 
If we set the weights to have an unbiased estimate, 
we achieve an accurate $R_G$ but suffer a huge regret in $R_D$.
On the other hand, we can store the columns unweighted to have small $R_D$
but large $R_G$. This could be potentially fixed if we knew how to remove less
important columns from dictionary to gain some slack in~$R_D$.

We illustrate this problem with following simple scenario. The adversary always presents
to the learner the same point~$\bx$ (with associated $\featvec$), but for the loss it alternates between
$\ell_{2t}(\bw_t) = (C - \featvec^\transp\bw_t)^2$ on even steps
and $\ell_{2t+1}(\bw_t) = (-C - \featvec^\transp\bw_t)^2$ on odd steps.
Then, $\sigma_t = \sigma = 1/(8C^2)$, and we have a gradient that always points
in the same $\featvec$ direction, but switches sign at each step. The optimal solution in hindsight is asymptotically
$\bw = \bzero$ and let 
this be also our starting point~$\bw_0$.
We also set $\eta_t = \sigma$, since this is what ONS would do, and
$\alpha = 1$ for simplicity.

For this scenario, we can compute several useful
quantities in closed form, in particular, $R_G$ and $R_D$, 
\begin{small}
\begin{align*}
R_G & \leq \sum_{t=1}^T \frac{\dot{g}_t^2}{\sum_{s=1}^t\dot{g}_s^2\sigma + \alpha}
\leq \sum_{t=1}^T\frac{C^2}{C^2\sigma t + \alpha} \leq \bigotime(\log T),
\\
R_D & = \sum\nolimits_{s=1}^t (\eta_t - \sigma)(\bw_t^{\transp}\bg_t)^2 = 0.
\end{align*}
\end{small}Note that although the matrix $\bA_t$ is rank 1 at each time step, vanilla ONS
does not take advantage of this easy data, and would store it all with a $\bigotime(t^2)$ space
 in  KOCO.

As for the sketched versions of ONS, sketching using FD~\cite{luo_efficient_2016} would adapt to this situation, and only store a single copy of
$\bg_t =\bg$, achieving the desired regret with a much smaller space. Notice
that in this example, the losses $\ell_t$ are effectively strongly convex,
and even basic gradient descent with a stepsize $\eta_t = 1/t$ would achieve
logarithmic regret \cite{zhu_online_2015} with even smaller space.
On the other hand, we show how the dictionary-based sketching has difficulties in minimizing the regret bound from
Prop.\,\ref{prop:exact-ons-regret} in our simple scenario.
In particular, consider an arbitrary (possibly randomized) algorithm that is allowed
only to reweight atoms in the dictionary 
and not to create new ones (as FD). In our example, this translates to choosing
a schedule of weights $w_s$ and set
$\wt{\bA}_t = \sum_{s=1}^t w_s\wb{\featvec}_s\wb{\featvec}_s = W_t\wb{\featvec}\wb{\featvec}$
with total weight $W = W_T = \sum_{s=1}^T w_s$ and space complexity equal to the number of non-zero weights $B = |\{w_s \neq 0\}|$.
We can show that there is no schedule for this specific class of algorithms with good performance
due to the following three conflicting goals.
\vspace{-0.12in}
\begin{itemize}
\itemsep0em 
\item[(1)] To mantain $R_G$ small, $\sum_{s=1}^t w_s$ should be as large as possible, as early as possible.
\item[(2)] To mantain $R_D$ small, we should choose weights $w_t > 1$ as few times as possible, since we accumulate $\max\{w_t - 1,0\}$ regret every time.
\item[(3)] To mantain the space complexity small, we should choose only a few $w_t \neq 0$.
\end{itemize}
\vspace{-0.12in}

\todomiout{maybe make (2) more precise, since this is still from the bound on the instantaneous regret? }
To enforce goal (3), we must choose a schedule with no more than $B$ non-zero entries.
Given the budget $B$, to satisfy goal (2) we should use all the $B$ budget in order
to exploit as much as possible the $\max\{w_t - 1,0\}$ in $R_D$, or in other words
we should use exactly $B$ non-zero weights, and none of these should be smaller than 1. Finally,
to minimize $R_G$ we should raise the sum $\sum_{s=1}^t w_s$ as quickly as possible,
settling on a schedule where $w_1 = W - B$ and $w_s = 1$ for all the other $B$
weights. It easy to see that if we want logarithmic $R_G$, $W$ needs to grow as
$T$, but doing so with a logarithmic $B$ would make $R_D = T-B = \Omega(T)$.
Similarly, keeping $W=B$ in order to reduce $R_D$ would increase $R_G$.
In particular notice, that the issue does not go away even if we know the RLS
perfectly, because the same reasoning applies. This simple example suggests that dictionary-based sketching methods, which are very successful in batch scenarios, may actually fail in achieving logarithmic regret in online optimization.

\todomiout{confusing $\bw$ and $w$ notation}
This argument raises the question on how to design alternative sketching methods for the second-order KOCO. 
A first approach, discussed above, is to reduce the dictionary size dropping
columns that become less important later in the process,
without allowing the adversary to take advantage of this forgetting factor.
Another possibility is to deviate from the ONS approach and  $R_D + R_G$ regret decomposition.
Finally, as our counterexample in the simple scenario
hints, creating new atoms (either through projection or merging)
allows for better adaptivity, as shown by FD \cite{ghashami2016frequent} based
methods in LOCO.
However, the kernelization of FD does not appear to be straighforward.
The most recent step in this direction (in particular, for kernel PCA)
is only able to deal with finite feature expansions \cite{ghashami2016streaming}
and therefore its application to kernels is limited.

{\small
\vspace{-0.15in}
\paragraph{\small Acknowledgements}
\label{sec:Acknowledgements}
The research presented was supported by French Ministry of
Higher Education and Research, Nord-Pas-de-Calais Regional Council and French National Research Agency projects ExTra-Learn (n.ANR-14-CE24-0010-01) and BoB (n.ANR-16-CE23-0003) }

\setlength{\bibsep}{5.2pt plus 20em}

\bibliographystyle{icml2017}

\newpage

\onecolumn

\appendix

\makeatletter{}
\section{Preliminary results}\label{app:prem-res}
We begin with a generic linear algebra identity that is be used throughout
our paper.
\begin{proposition}\label{prop:split-for-efficiency}
For any $\bX \in \Real^{n \times m}$ matrix and $\alpha > 0$,
\begin{align*}
\bX\bX^\transp(\bX\bX^\transp + \alpha\bI)^{-1}
&= \bX(\bX^\transp\bX + \alpha\bI)^{-1}\bX^\transp
\end{align*}
and
\begin{align*}
(\bX\bX^\transp + \alpha\bI)^{-1}
&= \frac{1}{\alpha}\alpha\bI(\bX\bX^\transp + \alpha\bI)^{-1}\\
&= \frac{1}{\alpha}(\bX\bX^\transp - \bX\bX^\transp + \alpha\bI)(\bX\bX^\transp + \alpha\bI)^{-1}\\
&= \frac{1}{\alpha}(\bI - \bX\bX^\transp(\bX\bX^\transp + \alpha\bI)^{-1})\\
&= \frac{1}{\alpha}(\bI - \bX(\bX^\transp\bX + \alpha\bI)^{-1}\bX^\transp).
\end{align*}

\end{proposition}

\begin{proposition}\label{prop:norm-to-lowner-guar}
For any matrix or linear operator $\bX$, if a selection matrix $\selmatrix$
satisfies
\begin{align*}
    \normsmall{(\bX\bX^\transp + \alpha\bI)^{-1/2}(\bX\bX^\transp - \bX\selmatrix\selmatrix^\transp\bX^\transp)(\bX\bX^\transp + \alpha\bI)^{-1/2}} \leq \varepsilon,
\end{align*}
we have
\begin{align*}
(1-\epsilon)\bX_t\bX_t^\transp - \varepsilon\alpha\bI
\preceq \bX_t\selmatrix_t\selmatrix_t^\transp\bX_t^\transp
\preceq (1+\epsilon)\bX_t\bX_t^\transp + \varepsilon\alpha\bI.
\end{align*}
\end{proposition}

\begin{proposition}\label{prop:accuracy-primal-dual-eq}
Let $\kermatrix_t = \bU\bsym{\Lambda}\bU^\transp$ and $\featkermatrix_t = \bV\bsym{\Sigma}\bU^\transp$, then
\begin{align*}
&\normsmall{(\featkermatrix_t\featkermatrix_t^\transp + \alpha\bI)^{-1/2}\featkermatrix_t(\bI - \selmatrix_s\selmatrix_s^\transp)\featkermatrix_t^\transp(\featkermatrix_t\featkermatrix_t^\transp + \alpha\bI)^{-1/2}}\\
&= \normsmall{(\bsym{\Sigma}\bsym{\Sigma}^\transp + \alpha\bI)^{-1/2}\bsym{\Sigma}\bU^\transp(\bI - \selmatrix_s\selmatrix_s^\transp)\bU\bsym{\Sigma}^\transp(\bsym{\Sigma}\bsym{\Sigma}^\transp + \alpha\bI)^{-1/2}}\\
&= \normsmall{(\bsym{\Lambda} + \alpha\bI)^{-1/2}\bsym{\Lambda}^{1/2}\bU^\transp(\bI - \selmatrix_s\selmatrix_s^\transp)\bU(\bsym{\Lambda}^{1/2})^\transp(\bsym{\Lambda} + \alpha\bI)^{-1/2}}\\
&= \normsmall{(\kermatrix_t + \alpha\bI)^{-1/2}\kermatrix_t^{1/2}(\bI - \selmatrix_s\selmatrix_s^\transp)\kermatrix_t^{1/2}(\kermatrix_t + \alpha\bI)^{-1/2}}.
\end{align*}
\end{proposition}

We also use the following concentration inequality for martingales.
        \begin{proposition}[\citealp{tropp2011freedman}, Thm.\,1.2]\label{prop:matrix-freedman}
Consider a matrix martingale $\{ \:Y_k : k = 0, 1, 2, \dots \}$ whose values are self-adjoint matrices with dimension $d$ and let $\{ \:X_k : k = 1, 2, 3, \dots \}$ be the difference sequence.  Assume that the difference sequence is uniformly bounded in the sense that
\begin{align*}
 \normsmall{\:X_k}_2  \leq R
\quad\text{almost surely}
\quad\text{for $k = 1, 2, 3, \dots$}.
\end{align*}
Define the predictable quadratic variation process of the martingale as
\begin{align*}
\:{W}_k := \sum_{j=1}^k \expectedvalue \left[ \:X_j^2 \condbar \{\:X_{s}\}_{s=0}^{j-1} \right],
\quad\text{for $k = 1, 2, 3, \dots$}.
\end{align*}
Then, for all $\varepsilon \geq 0$ and $\sigma^2 > 0$,
\begin{align*}
\probability\left( \exists k \geq 0 : \normsmall{\:Y_k}_2 \geq \varepsilon \ \cap\ 
        \normsmall{ \:W_{k} } \leq \sigma^2 \right)
	\leq 2d \cdot \exp \left\{ - \frac{ \varepsilon^2/2 }{\sigma^2 + R\varepsilon/3} \right\}.
\end{align*}
\end{proposition}

\begin{proposition}[\citealp{calandriello_disqueak_2017}, App.\,D.4]\label{prop:hoeff-sum-bern}
Let $\{z_s\}_{s=1}^t$ be independent
Bernoulli random variables, each with success probability $p_s$,
and denote their sum as $d = \sum_{s=1}^t p_s \geq 1$. Then,\footnote{This is a simple variant of Chernoff bound where the Bernoulli random variables are not identically distributed.}
\begin{align*}
\probability\left(\sum_{s=1}^t z_s \geq 3d\right) \leq \exp\{-3d(3d - (\log(3d)+1))\} \leq \exp\{-2d\}
\end{align*}
\end{proposition}
 
\makeatletter{}
\section{Proofs for Section~\ref{sec:exact-kern-ons}}\label{app:exact-kern-ons}

\begin{proof}[Proof of Lem.~\ref{lem:exact-kern-ons-u-reform}]
We begin by applying the definition of $\bu_{t+1}$ and collecting $\bA_t^{-1}$,
which can always be done since, for $\alpha > 0$, $\bA_t$ is invertible,
\begin{align*}
&\bu_{t+1} = \bw_t - \bA_t^{-1}\bg_t
=\bA_t^{-1}(\bA_t\bw_t - \bg_t).
\end{align*}
We focus now on the last term and use the definition of $\bA_t$,
\begin{align*}
\bA_t\bw_t - \bg_t &= \bA_{t-1}\bw_t + \eta_t\bg_t\bg_t^\transp\bw_t- \bg_t\\
&= \bA_{t-1}\bu_t - \bA_{t-1}\br_t + (\sqrt{\eta_t}\bg_t^\transp\bw_t - 1/\sqrt{\eta_t})\wb{\featvec}_t.
\end{align*}
Looking at $\bA_{t-1}\br_t$ and using the assumption $\dot{g}_t \neq 0$,
\begin{align*}
\bA_{t-1}\br_t &= \frac{h(\featvec_t^\transp\bu_t)}{\featvec_t^\transp\bA_{t-1}^{-1}\featvec_t}\bA_{t-1}\bA_{t-1}^{-1}\featvec_t\\
&=\frac{h(\featvec_t^\transp\bu_t)}{\featvec_t^\transp\bA_{t-1}^{-1}\featvec_t}\frac{\dot{g}_t^2\eta_t}{\dot{g}_t^2\eta_t}\featvec_t\\
&=\frac{\dot{g}_t\sqrt{\eta_t}h(\featvec_t^\transp\bu_t)}{\wb{\featvec}_t^\transp\bA_{t-1}^{-1}\wb{\featvec}_t}\wb{\featvec}_t.
\end{align*}
Putting together all three terms, and using the fact that
$\bg_t^\transp\bw_t = \dot{g}_t\featvec_t\bw_t = \dot{g}_t\wh{y}_t$
and denoting $b_t = [\bb_t]_t$ we have
\begin{align*}
\bu_{t+1} &= \bA_t^{-1}(\bA_t\bw_t - \bg_t)\\
&=\bA_t^{-1}(\bA_{t-1}\bu_t + b_t\wb{\featvec}_t)\\
&= \bA_t^{-1}(\bA_{t-1}(\bw_{t-1} - \bA_{t-1}^{-1}\bg_{t-1}) + b_t\wb{\featvec}_t)\\
&= \bA_t^{-1}(\bA_{t-1}\bw_{t-1} - \bg_{t-1} + b_t\wb{\featvec}_t)\\
&= \bA_t^{-1}(\bA_{t-2}\bw_{t-2} - \bg_{t-2} + b_{t-1}\wb{\featvec}_{t-1} + b_t\wb{\featvec}_t)\\
&= \bA_t^{-1}(\bA_{0}\bw_{0} + \sum\nolimits_{s=1}^{t}b_s\wb{\featvec}_s).
\end{align*}
\end{proof}

\begin{proof}[Proof of Lem.~\ref{lem:exact-computable-reformulation}]
Throughout this proof, we make use of the linear algebra identity from
Prop.\,\ref{prop:split-for-efficiency}.
We begin with the reformulation of $[\bb_t]_t$. In particular, the only term that we
need to reformulate is
\begin{align*}
&\wb{\featvec}_t\bA_{t-1}^{-1}\wb{\featvec}_t
= \wb{\featvec}_t(\wb{\featkermatrix}_{t-1}\wb{\featkermatrix}_{t-1}^\transp + \alpha\bI)^{-1}\wb{\featvec}_t\\
&= \frac{1}{\alpha}\wb{\featvec}_t(\bI - \wb{\featkermatrix}_{t-1}(\wb{\featkermatrix}_{t-1}^\transp\wb{\featkermatrix}_{t-1} + \alpha\bI)^{-1}\wb{\featkermatrix}_{t-1}^\transp)\wb{\featvec}_t\\
&= \frac{1}{\alpha}(\wb{\featvec}_t^\transp\wb{\featvec}_t - \wb{\featvec}_t^\transp\wb{\featkermatrix}_{t-1}(\wb{\featkermatrix}_{t-1}^\transp\wb{\featkermatrix}_{t-1} + \alpha\bI)^{-1}\wb{\featkermatrix}_{t-1}^\transp\wb{\featvec}_t)\\
&= \frac{1}{\alpha}(k_{t,t} - \wb{\bk}_{[t-1],t}^\transp(\wb{\kermatrix}_{t-1} + \alpha\bI)^{-1}\wb{\bk}_{[t-1],t}).
\end{align*}
For $\wb{y}_t$, we have
\begin{align*}
\wb{y}_{t} &= \featvec_{t}^\transp\bu_{t} = \featvec_{t}^\transp\bA_{t-1}^{-1}\wb{\featkermatrix}_{t-1}\bb_{t-1}\\
&= \featvec_{t}^\transp(\wb{\featkermatrix}_{t-1}\wb{\featkermatrix}_{t-1}^\transp + \alpha\bI)^{-1}\wb{\featkermatrix}_{t-1}\bb_{t-1}\\
&= \frac{1}{\alpha}\featvec_{t}^\transp(\bI - \wb{\featkermatrix}_{t-1}(\wb{\featkermatrix}_{t-1}^\transp\wb{\featkermatrix}_{t-1} + \alpha\bI)^{-1}\wb{\featkermatrix}_{t-1}^\transp)\wb{\featkermatrix}_{t-1}\bb_{t-1}\\
&= \frac{1}{\alpha}\featvec_{t}^\transp\featkermatrix_{t-1}\bD_{t-1}(\bb_{t-1} - (\wb{\kermatrix}_{t-1} + \alpha\bI)^{-1}\wb{\kermatrix}_{t-1}\bb_{t-1})\\
&= \frac{1}{\alpha}\bk_{[t-1],t}^\transp\bD_{t-1}(\bb_{t-1} - (\wb{\kermatrix}_{t-1} + \alpha\bI)^{-1}\wb{\kermatrix}_{t-1}\bb_{t-1}).
\end{align*}
\end{proof}

\begin{proof}[Proof of Lem.~\ref{lem:bound-online-tau}]
We  prove the lemma for a generic kernel $\kerfunc$ and kernel matrix $\kermatrix_T$.
Then, Lem.\,\ref{lem:bound-online-tau} simply follows by applying the proof to
$\wb{\kerfunc}$ and $\wb{\kermatrix}_T$.
From the definition of $\tau_{t,t}$ we have
\begin{align*}
\sum_{t=1}^{T} \tau_{t,t}
&= \sum_{t=1}^{T} \featvec_t^\transp(\featkermatrix_t\featkermatrix_t^\transp + \alpha\:I)^{-1}\featvec_t
= \sum_{t=1}^{T} (\featvec_t^\transp/\sqrt{\alpha})\left(\featkermatrix_t\featkermatrix_t^\transp/\alpha + \:I\right)^{-1}(\featvec_t/\sqrt{\alpha})
\leq \log(\Det(\featkermatrix_T\featkermatrix_T^\transp/\alpha + \:I)),
\end{align*}
where the last passage is proved by \citet{hazan_logarithmic_2006}.
Using Sylvester's determinant identity,
\begin{align*}
\Det(\featkermatrix_T\featkermatrix_T^\transp/\alpha + \:I) = \Det(\featkermatrix_T^\transp\featkermatrix_T/\alpha + \:I) = \prod_{t=1}^T (\lambda_t/\alpha + 1),
\end{align*}
where $\lambda_t$ are the eigenvalues of $\featkermatrix_T^\transp\featkermatrix_T = \kermatrix_T$. Then,
\begin{align*}
\sum_{t=1}^{T} \tau_{t,t}
\leq \log\left(\prod\nolimits_{t=1}^T (\lambda_t/\alpha + 1)\right)
= \sum\nolimits_{t=1}^T \log(\lambda_t/\alpha + 1).
\end{align*}
We can decompose this as
\begin{align*}
\sum_{t=1}^T\log(\lambda_t/\alpha + 1) &= \sum_{t=1}^T\log(\lambda_t/\alpha + 1)\left(\frac{\lambda_t/\alpha + 1}{\lambda_t/\alpha + 1}\right)\\
&= \sum_{t=1}^T\log(\lambda_t/\alpha + 1)\frac{\lambda_t/\alpha}{\lambda_t/\alpha + 1}
+\sum_{t=1}^T\frac{\log(\lambda_t/\alpha + 1)}{\lambda_t/\alpha + 1}\\
&\leq \log(\normsmall{\kermatrix_T}/\alpha + 1)\sum_{t=1}^T\frac{\lambda_t}{\lambda_t + \alpha}
+\sum_{t=1}^T\frac{\log(\lambda_t/\alpha + 1)}{\lambda_t/\alpha + 1}\\
&\leq \log(\normsmall{\kermatrix_T}/\alpha + 1)\deff^T(\alpha)
+\sum_{t=1}^T\frac{(\lambda_t/\alpha + 1) - 1}{\lambda_t/\alpha + 1}\\
&= \log(\normsmall{\kermatrix_T}/\alpha + 1)\deff^T(\alpha)
+\deff^T(\alpha),
\end{align*}
where the first inequality is due to $\normsmall{\kermatrix_T} \geq \lambda_t$
for all $t$ and the monotonicity of $\log(\cdot)$,
and the second inequality is due to
$\log(x) \leq x - 1$.
\end{proof}

\begin{proof}[Proof of Thm~\ref{thm:kernel-ons-regret}]
We need to bound $R_T(\bw^{*})$, and we use Prop.\,\ref{prop:exact-ons-regret}.
For $R_D$ nothing changes from the parametric case, and we use Asm.\,\ref{ass:scalar-lipschitz} and
the definition of the set $\feasibleset$ to bound
\begin{align*}
R_D = \sum\nolimits_{t=1}^T (\eta_t - \sigma_t)\dot{g}_t^2(\featvec_t^\transp(\bw_t - \bw))^2
\leq \sum\nolimits_{t=1}^T (\eta_t - \sigma)L^2(|\featvec_t^\transp\bw_t| + |\featvec_t^\transp\bw|)^2
\leq 4L^2C^2\sum\nolimits_{t=1}^T (\eta_t - \sigma).
\end{align*}

For $R_G$, we reformulate
\begin{align*}
\sum\nolimits_{t=1}^T \bg_t^\transp\bA_t^{-1}\bg_t
&= \sum\nolimits_{t=1}^T\frac{\eta_t}{\eta_t} \bg_t^\transp\bA_t^{-1}\bg_t
= \sum\nolimits_{t=1}^T\frac{1}{\eta_t} \wb{\featvec}_t^\transp\bA_t^{-1}\wb{\featvec}_t\\
&\leq \frac{1}{\eta_T} \sum\nolimits_{t=1}^T\wb{\featvec}_t^\transp\bA_t^{-1}\wb{\featvec}_t
= \frac{1}{\eta_T} \sum\nolimits_{t=1}^T\wb{\tau}_{t,t} = \wb{d}_{\text{onl}}(\alpha)/\eta_T
\leq \frac{\bdeff^T(\alpha)}{\eta_T}(1 + \log(\normsmall{\wb{\kermatrix}_T}/\alpha + 1)),
\end{align*}
where $\bdeff^T(\alpha)$ and $\bkermatrix_T$ are computed using the rescaled
kernel $\wb{\kerfunc}$.

Let us remind ourselves the definition $\bD = \Diag\left(\{\dot{g}_t\sqrt{\eta_t}\}_{t=1}^T\right)$. Since $\eta_t \neq 0$ and $\dot{g}_t \neq 0$ for all $t$, $\bD$ is invertible and we have $\lambda_{\min}(\bD^{-2}) = \min_{t=1}^T 1/(\dot{g}_t^2\eta_t) \geq 1/(L^2\eta_1)$. For simplicity, we assume $\eta_1 = \sigma$, leaving the case $\eta_1 = 1/1 = 1$ as a special case. We derive
\begin{align*}
\bdeff^T(\alpha)
&= \Tr(\wb{\kermatrix}_T(\wb{\kermatrix}_T + \alpha\bI)^{-1})\\
&= \Tr(\bD\kermatrix_T\bD(\bD\kermatrix_T\bD + \alpha\bD\bD^{-2}\bD)^{-1})\\
&= \Tr(\bD\kermatrix_T\bD\bD^{-1}(\kermatrix_T + \alpha\bD^{-2})^{-1}\bD^{-1})\\
&= \Tr(\kermatrix_T\bI(\kermatrix_T + \alpha\bD^{-2})^{-1}\bD^{-1}\bD)\\
&= \Tr(\kermatrix_T(\kermatrix_T + \alpha\bD^{-2})^{-1})\\
&\leq \Tr(\kermatrix_T(\kermatrix_T + \alpha\lambda_{\min}(\bD^{-2})\bI)^{-1})\\
&\leq \Tr\left(\kermatrix_T\left(\kermatrix_T + \frac{\alpha}{\sigma L^2}\bI\right)^{-1}\right)
=\deff^{T}\left( \alpha/(\sigma L^2) \right).
\end{align*}

Similarly,
\begin{align*}
&\log(\normsmall{\wb{\kermatrix}_T}/\alpha + 1)
\leq \log(\Tr(\wb{\kermatrix}_T)/\alpha + 1)
\leq \log(\sigma L^2 \Tr(\kermatrix_t)/\alpha + 1)
\leq \log(\sigma L^2 T/\alpha + 1)
\leq \log(2\sigma L^2 T/\alpha),
\end{align*}
since $\Tr(\kermatrix_t) = \sum_{t=1}^T k_{t,t} = \sum_{t=1}^T \featvec_t^\transp\featvec_t \leq \sum_{t=1}^T 1 = T$.
\end{proof}

\makeatletter{}
\section{Proofs for Section \ref{sec:online-squeak}}\label{app:online-squeak}

\begin{proof}[Proof of Thm.~\ref{thm:online-squeak}]
We derive the proof for a generic $\kerfunc$
with its induced $\featvec_t = \varphi(\bx_t)$
and $\kermatrix_t$. Then, \sketchkons (Alg.\,\ref{alg:ons-sketch})
applies this proof to the rescaled $\wb{\featvec}_t$ and $\wb{\kermatrix}_t$.

Our goal is to prove that Alg.\,\ref{alg:onsqueak} generates accurate and
small dictionaries at all time steps $t \in [T]$. More formally, a dictionary
$\coldict_s$ is $\varepsilon$-accurate w.r.t.\,$\dataset_t$ when
\begin{align*}
\normsmall{(\featkermatrix_t\featkermatrix_t^\transp + \alpha\bI)^{-1/2}\featkermatrix_t(\bI - \selmatrix_s\selmatrix_s^\transp)\featkermatrix_t^\transp(\featkermatrix_t\featkermatrix_t^\transp + \alpha\bI)^{-1/2}}
&= \normsmall{(\kermatrix_t + \alpha\bI)^{-1/2}\kermatrix_t^{1/2}(\bI - \selmatrix_s\selmatrix_s^\transp)\kermatrix_t^{1/2}(\kermatrix_t + \alpha\bI)^{-1/2}} \leq \varepsilon,
\end{align*}
where we used Prop.\,\ref{prop:accuracy-primal-dual-eq} to move from feature to
primal space.

We also introduce the projection operators,
\begin{align*}
\pvec_{t,i} :=& ((\kermatrix_t + \alpha\bI)^{-1}\kermatrix_t)^{1/2}\be_{t,i}\\
\bP_t :=& (\kermatrix_t + \alpha\bI)^{-1/2}\kermatrix_t^{1/2}\kermatrix_t^{1/2}(\kermatrix_t + \alpha\bI)^{-1/2}
= \sum_{s=1}^t\pvec_{t,s}\pvec_{t,s}^\transp = \pmat_t\pmat_t^\transp\\
\wt{\bP}_t :=& (\kermatrix_t + \alpha\bI)^{-1/2}\kermatrix_t^{1/2}\selmatrix_t\selmatrix_t^\transp\kermatrix_t^{1/2}(\kermatrix_t + \alpha\bI)^{-1/2}
= \sum_{s=1}^t\frac{z_t}{\wt{p}_t}\pvec_{t,s}\pvec_{t,s}^\transp = \pmat_t\selmatrix_t\selmatrix_t^\transp\pmat_t^\transp,
\end{align*}
where the $z_t$ variables are the $\{0,1\}$ random variables sampled by Alg.\,\ref{alg:onsqueak}.
Note that with this notation we have
\begin{align*}
\normsmall{\pvec_{t,i}\pvec_{t,i}^\transp}
&= \normsmall{((\kermatrix_t + \alpha\bI)^{-1}\kermatrix_t)^{1/2}\be_{t,i}\be_{t,i}^\transp(\kermatrix_t(\kermatrix_t + \alpha\bI)^{-1})^{1/2}}\\
&= \be_{t,i}^\transp((\kermatrix_t + \alpha\bI)^{-1}\kermatrix_t)^{1/2}(\kermatrix_t(\kermatrix_t + \alpha\bI)^{-1})^{1/2}\be_{t,i} = \be_{t,i}^\transp(\kermatrix_t + \alpha\bI)^{-1}\kermatrix_t\be_{t,i} = \tau_{t,i}.
\end{align*}
We can now formalize the event ``some of the guarantees of Alg.\,\ref{alg:onsqueak} do not hold'' and bound the probability of this event. In particular,
let 
\begin{align*}
\bY_t := \wt{\bP}_t - \bP_t  = \sum_{s=1}^t\left(\frac{z_t}{\wt{p}_t} - 1\right)\pvec_{t,s}\pvec_{t,s}^\transp. 
\end{align*}
We want to show
\begin{align*}
    &\probability\bigg(\exists  t\in [T]: \underbrace{\normsmall{\bY_t} \geq \varepsilon}_{A_t} \;\cup\; \underbrace{\sum\nolimits_{s=1}^t z_{t} \geq 3 \beta\donl^t(\alpha)}_{B_t}\bigg) \leq \delta,
\end{align*}
where event $A_t$ refers to the case when the intermediate dictionary $\coldict_t$ fails to accurately approximate $\kermatrix_t$ at some step $t \in [T]$ and event $B_t$ considers the case when the memory requirement is not met (i.e., too many columns are kept in a dictionary $\coldict_t$ at a certain time $t\in[T]$).

\textbf{Step 1: Splitting the problem.}
We can conveniently decompose the previous joint (negative) event into two separate conditions as
\begin{align*}    \probability\bigg(\bigcup_{t = 1}^{T} \big(A_t \cup B_t\big)\bigg)
    &=\probability\left(\left\{\bigcup_{t = 1}^{T} A_t\right\} \right) +\probability\left( \left\{ \bigcup_{t = 1}^{T} B_t\right\}\right) -\probability\left(\left\{\bigcup_{t = 1}^{T} A_t\right\} \cap \left\{ \bigcup_{t = 1}^{T} B_t\right\}\right) \\
    &=\probability\left(\left\{\bigcup_{t = 1}^{T} A_t\right\} \right) +\probability\left(\left\{\bigcup_{t = 1}^{T} B_t\right\} \cap \left\{ \bigcup_{t = 1}^{T} A_t\right\}^{\complement}\right) 
    =\probability\left(\left\{\bigcup_{t = 1}^{T} A_t\right\} \right) +\probability\left(\left\{\bigcup_{t = 1}^{T} B_t\right\} \cap \left\{ \bigcap_{t = 1}^{T} A_t^{\complement}\right\}\right) \\
    &=\probability\left(\left\{\bigcup_{t = 1}^{T} A_t\right\} \right) +\probability\left(\bigcup_{t = 1}^{T} \left\{B_t \cap \left\{ \bigcap_{t' = 1}^{T} A_{t'}^{\complement}\right\}\right\}\right).
\end{align*}
Applying this reformulation and a union bound, we obtain
\begin{align*}
    \probability\bigg(\exists  t\in [T]: \normsmall{\bY_t} &\geq \varepsilon \;\cup\; \sum\nolimits_{s=1}^t z_{t} \geq 3 \beta\donl^t(\alpha)\bigg)\\
     &\leq \sum_{t = 1}^{T} \probability\left( \normsmall{\bY_t} \geq \epsilon\right) +\sum_{t = 1}^{T}\probability\left(\sum_{s=1}^t z_{s} \geq 3\beta\donl^t(\alpha) \cap \left\{\forall \; t' \in \{1, \dots, t\} :  \normsmall{\bY_t} \leq \varepsilon \right\}\right).
\end{align*}
To conclude the proof, we show in Step 2 and 3, that each of the failure events happens with probability less than $\frac{\delta}{2T}$.

\textbf{Step 2: Bounding the accuracy.}
We first point out that dealing with $\bY_t$ is not trivial since the process $\{\bY_t\}_{t=1}^T$ is composed by matrices of different size, that cannot be directly compared.
Denote with $\selmatrix_s^t$ the matrix constructed by (1) taking $\selmatrix_s$
and adding $t-s$ rows of zeros to its bottom to extend it, and (2) adding $t-s$ indicator columns $\be_{t,i}$ for all $i > s$.
We begin by reformulating $\bY_t$ as a random process $\bY_0^t, \bY_1^t, \dots , \bY_t^t$ with differences $\bX_s^t$
defined as
\begin{align*}
\bX^t_s = \left(\frac{z_s}{\wt{p}_s} - 1\right)\pvec_{t,s}\pvec_{t,s}^\transp, && \bY_k^t
= \sum_{s=1}^{k} \bX_s^t = \sum_{s=1}^k \left(\frac{z_s}{\wt{p}_s} - 1\right)\pvec_{t,s}\pvec_{t,s}^\transp
= \pmat_t(\selmatrix_k^t(\selmatrix_k^t)^\transp - \bI)\pmat_t^\transp.
\end{align*}
We introduce the \textit{freezing} probabilities
\begin{align*}
\wb{p}_s = \wt{p}_s \cdot \indfunc\{\normsmall{\bY_{s-1}^t} < \varepsilon\} + 1 \cdot \indfunc\{\normsmall{\bY_{s-1}^t} \geq \varepsilon\}
\end{align*}
and the associated process $\wb{\bY}_s^t$ based on the coin flips $\wb{z}_s$
performed using $\wb{p}_{s}$ instead of the original $\wt{p}_s$ as in Alg.\,\ref{alg:onsqueak}. In other words, this process is such that if at any time $s-1$ the accuracy condition is not met, then for all steps from $s$ on the algorithm stops updating the dictionary. We also define $\wb{\bY}_t = \wb{\bY}_t^t$. Then we have
\begin{align*}
\probability\left( \normsmall{\bY_t} \geq \epsilon\right)
\leq \probability\left( \normsmall{\wb{\bY}_t} \geq \epsilon\right),
\end{align*}
so we can simply bound the latter to bound the former.
To show the usefulness of the freezing process, consider the step $\wb{s}$
where the process froze, or more formally define $\wb{s}$ as the step
where $\normsmall{\bY_{\wb{s}}^t} < \varepsilon$ and $\normsmall{\bY_{\wb{s}+1}^t} \geq \varepsilon$.
Then for all $s \leq \wb{s}$, we can combine Prop.\,\ref{prop:norm-to-lowner-guar},
the definition of $\pmat_t$,
and the guarantee that $\normsmall{\bY_{s}^t} < \varepsilon$ to obtain
\begin{align*}
    \featkermatrix_s\selmatrix_s\selmatrix_s^\transp\featkermatrix_s
    \preceq \featkermatrix_t\selmatrix_s^t(\selmatrix_s^t)^\transp\featkermatrix_t
    \preceq \featkermatrix_{t}\featkermatrix_{t}^\transp + \varepsilon(\featkermatrix_t\featkermatrix_t^\transp + \gamma\bI),
\end{align*}
where in the first inequality we used the fact that $\bS_s^t$ is simply obtained by bordering $\bS_s$. \todoaout{This step may be explained in more detail.}
Applying the definition of $\wt{p}_s$, when $\wt{p}_s < 1$ we have
\begin{align*}
\wb{p}_s = \wt{p}_s = \beta\atau_{s,s} &= \beta(1 + \varepsilon)\featvec_s^\transp(\featkermatrix_{s}\selmatrix_{s}\selmatrix_{s}^\transp\featkermatrix_{s}^\transp + \gamma \bI)^{-1}\featvec_s\\
&\geq \beta(1 + \varepsilon)\featvec_s^\transp(\featkermatrix_{s}\featkermatrix_{s}^\transp + \varepsilon(\featkermatrix_t\featkermatrix_t^\transp + \gamma\bI) + \gamma \bI)^{-1}\featvec_s\\
&= \beta(1 + \varepsilon)\frac{1}{1+\varepsilon}\featvec_s^\transp(\featkermatrix_t\featkermatrix_t^\transp + \gamma \bI)^{-1}\featvec_s,
= \beta\tau_{t,s},
\end{align*}
which shows that our estimates of RLS are upper bounds on the true values.

\todomi{the $\gamma$ should be $\alpha$ above}

From this point onwards we focus on a specific $t$, and omit the index
from $\wb{\bY}_s^t$, $\wb{\bX}_s^t$ and $\bv_{t,s}$.
We can now verify that $\wb{\bY}_s$ is a martingale, by showing that $\wb{\bX}_s$ is
zero mean.
Denote with $\filtration_k = \{\wb{\bX}_s\}_{s=1}^k$ the filtration of the process.
When $\wb{p}_s = 1$, either because $\beta\wt{p}_s \geq 1$ or
because the process is frozen, we have $\wb{\bX}_s = \bsym{0}$
and the condition is satisfied.
Otherwise, we have
\begin{align*}
\expectedvalue\left[\wb{\bX}_s \condbar \filtration_{s-1}\right]
= \expectedvalue\left[\left(\frac{\wb{z}_s}{\wb{p}_s} - 1\right)\pvec_{s}\pvec_{s}^\transp \condbar \filtration_{s-1}\right]
= \left(\frac{\expectedvalue\left[\wb{z}_s\condbar \filtration_{s-1}\right]}{\wb{p}_s} - 1\right)\pvec_{s}\pvec_{s}^\transp
= \left(\frac{\wb{p}_s}{\wb{p}_s} - 1\right)\pvec_{s}\pvec_{s}^\transp = \bsym{0},
\end{align*}
where we use the fact that $\wb{p}_s$ is fixed conditioned on $\filtration_{s-1}$ and its the (conditional) expectation of $\wb{z}_s$.
Since $\wb{\bY}_t$ is a martingale, we can use Prop.~\ref{prop:matrix-freedman}.
First, we find $R$. Again, when $\wb{p}_s = 1$ we have $\bX_s = \bsym{0}$ and
$R \geq 0$. Otherwise,
\begin{align*}
\normempty{\left(\frac{\wb{z}_s}{\wb{p}_s} - 1\right)\pvec_{s}\pvec_{s}^\transp}
\leq \left|\left(\frac{\wb{z}_s}{\wb{p}_s} - 1\right)\right|\normsmall{\pvec_{s}\pvec_{s}^\transp}
\leq \frac{1}{\wb{p}_s}\tau_{t,s}
\leq \frac{\tau_{t,s}}{\beta\tau_{t,s}}
=\frac{1}{\beta} := R.
\end{align*}
For the total variation, we expand
\begin{align*}
\wb{\bW}_t := \sum_{s=1}^t \expectedvalue\left[\wb{\bX}_s^2 \condbar \filtration_{s-1}\right]
&= \sum_{s=1}^t \expectedvalue\left[\left(\frac{\wb{z}_s}{\wb{p}_s} - 1\right)^2 \condbar \filtration_{s-1}\right] \pvec_{s}\pvec_{s}^\transp\pvec_{s}\pvec_{s}^\transp\\
&= \sum_{s=1}^t
\left(\expectedvalue\left[\frac{\wb{z}_s^2}{\wb{p}_s^2} \condbar \filtration_{s-1}\right]
- \expectedvalue\left[2\frac{\wb{z}_s}{\wb{p}_s} \condbar \filtration_{s-1}\right]
+ 1\right) \pvec_{s}\pvec_{s}^\transp\pvec_{s}\pvec_{s}^\transp\\
&= \sum_{s=1}^t
\left(\expectedvalue\left[\frac{\wb{z}_s}{\wb{p}_s^2} \condbar \filtration_{s-1}\right] - 1\right) \pvec_{s}\pvec_{s}^\transp\pvec_{s}\pvec_{s}^\transp
=\sum_{s=1}^t\left(\expectedvalue\left[\frac{\wb{z}_s}{\wb{p}_s^2} \condbar \filtration_{s-1}\right] - 1\right) \pvec_{s}\pvec_{s}^\transp\pvec_{s}\pvec_{s}^\transp\\
&=\sum_{s=1}^t\left(\frac{1}{\wb{p}_s} - 1\right) \pvec_{s}\pvec_{s}^\transp\pvec_{s}\pvec_{s}^\transp
=\sum_{s=1}^t\left(\frac{\pvec_{s}^\transp\pvec_{s}}{\wb{p}_s} - \pvec_{s}^\transp\pvec_{s}\right) \pvec_{s}\pvec_{s}^\transp
=\sum_{s=1}^t\left(\frac{\tau_{t,s}}{\wb{p}_s} - \tau_{t,s}\right) \pvec_{s}\pvec_{s}^\transp,
\end{align*}
where we used the fact that $\wb{z}_s^2 = \wb{z}_s$ and $\expectedvalue[\wb{z}_s | \filtration_{s-1}] = \wb{p}_s$.
We can now bound this quantity as
\begin{align*}
\normempty{\wb{\bW}_t} = \normempty{\sum_{s=1}^t \expectedvalue\left[\wb{\bX}_s^2 \condbar \filtration_{s-1}\right]}
&=\normempty{\sum_{s=1}^t\left(\frac{\tau_{t,s}}{\wb{p}_s} - \tau_{t,s}\right) \pvec_{s}\pvec_{s}^\transp}
\leq\normempty{\sum_{s=1}^t\frac{\tau_{t,s}}{\wb{p}_s}\pvec_{s}\pvec_{s}^\transp}
\leq\normempty{\sum_{s=1}^t\frac{\tau_{t,s}}{\beta\tau_{t,s}}\pvec_{s}\pvec_{s}^\transp}\\
&=\frac{1}{\beta}\normempty{\sum_{s=1}^t\pvec_{s}\pvec_{s}^\transp}
=\frac{1}{\beta}\normempty{\pmat_t\pmat_t^\transp}
=\frac{1}{\beta}\normempty{\bP_t}
\leq \frac{1}{\beta} := \sigma^2.
\end{align*}
Therefore, if we let $\sigma^2 = 1/\beta$ and $R = 1/\beta$, we have
\begin{align*}
\probability\left( \normsmall{\bY_t} \geq \epsilon\right)
&\leq \probability\left( \normsmall{\wb{\bY}_t} \geq \epsilon\right)
= \probability\left( \normsmall{\wb{\bY}_t} \geq \epsilon \cap \normsmall{\wb{\bY}_t} \leq \sigma^2\right)
+ \probability\left( \normsmall{\wb{\bY}_t} \geq \epsilon \cap \normsmall{\wb{\bY}_t} \geq \sigma^2\right)\\
&\leq \probability\left( \normsmall{\wb{\bY}_t} \geq \epsilon \cap \normsmall{\wb{\bY}_t} \leq \sigma^2\right)
+ \probability\left(\normsmall{\wb{\bY}_t} \geq \sigma^2\right)
\leq 2t\exp\left\{- \frac{\varepsilon^2}{2}\frac{1}{\frac{1}{\beta}(1 + \varepsilon/3)}\right\}
+ 0
\leq 2t\exp\left\{- \frac{\varepsilon^2\beta}{3}\right\}\cdot
\end{align*}

\textbf{Step 3: Bounding the space.}
We want to show that
\begin{align*}
\probability\left(\sum_{s=1}^t z_{s} \geq 3\beta\donl^t(\alpha) \cap \left\{\forall \; t' \in \{1, \dots, t\} :  \normsmall{\bY_t} \leq \varepsilon \right\}\right).
\end{align*}
Assume, without loss of generality, that for all $s \in [t]$ we have $\beta\tau_{s,s} \leq 1$,
and introduce the independent Bernoulli random variables $\wh{z}_s \sim \bernoullidist(\beta\tau_{s,s})$.
Thanks to the intersection with the event $\left\{\forall \; t' \in \{1, \dots, t\} :  \normsmall{\bY_t} \leq \varepsilon \right\}$, we know that all dictionaries $\coldict_s$ are $\varepsilon$-accurate,
and therefore for all $s$ we have $\wt{p}_s \leq \beta\atau_{s,s} \leq  \beta\tau_{s,s}$. Thus
$\wh{z}_s$ stochastically dominates $z_s$ and we have
\begin{align*}
\probability\left(\sum_{s=1}^t z_{s} \geq 3\beta\donl^t(\alpha) \cap \left\{\forall \; t' \in \{1, \dots, t\} :  \normsmall{\bY_t} \leq \varepsilon \right\}\right)
 \leq \probability\left(\sum_{s=1}^t \wh{z}_{s} \geq 3\beta\donl^t(\alpha)\right).
\end{align*}
Applying Prop.\,\ref{prop:hoeff-sum-bern} to $\sum_{s=1}^t \wh{z}_{s}$ and knowing that
$\sum_{s=1}^t p_t = \sum_{s=1}^t \beta\tau_{s,s} = \beta\donl^t(\alpha)$, we have
\begin{align*}
\probability\left(\sum_{s=1}^t \wh{z}_{s} \geq 3\beta\donl^t(\alpha)\right)
\leq \exp\{-3\beta\donl^t(\alpha)(3\beta\donl^t(\alpha) - (\log(3\beta\donl^t(\alpha)) + 1))\}
\leq \exp\{-2\beta\donl^t(\alpha)\}.
\end{align*}
Assuming $\donl^t(\alpha) \geq 1$, we have that $\exp\{-2\beta\donl^t(\alpha)\} \leq \exp\{-2\beta\} \leq \exp\{-\log((T/\delta)^2)\} \leq \delta^2/T^2 \leq \delta/(2T)$ as long as $2\delta \leq 2 \leq T$.
\end{proof}

\makeatletter{}
\section{Proofs for Section \ref{sec:skethed-ons}}\label{app:skethed-ons}

\begin{proof}[Proof of Lemma~\ref{lem:sketched-computable-reformulation}]
Through this proof, we make use of the linear algebra identity from
Prop.\,\ref{prop:split-for-efficiency}.
We begin with the reformulation of $\wt{b}_{i} = [\wt{\bb}_t]_i$. In particular, the only term that we
need to reformulate is
\begin{align*}
&\wb{\featvec}_t\wt{\bA}_{t-1}^{-1}\wb{\featvec}_t
= \wb{\featvec}_t(\wb{\featkermatrix}_{t-1}\selrmatrix_{t-1}\selrmatrix_{t-1}^\transp\wb{\featkermatrix}_{t-1}^\transp + \alpha\bI)^{-1}\wb{\featvec}_t\\
&= \frac{1}{\alpha}\wb{\featvec}_t(\bI - \wb{\featkermatrix}_{t-1}\selrmatrix_{t-1}(\selrmatrix_{t-1}^\transp\wb{\featkermatrix}_{t-1}^\transp\wb{\featkermatrix}_{t-1}\selrmatrix_{t-1} + \alpha\bI)^{-1}\selrmatrix_{t-1}^\transp\wb{\featkermatrix}_{t-1}^\transp)\wb{\featvec}_t\\
&= \frac{1}{\alpha}(\wb{\featvec}_t^\transp\wb{\featvec}_t - \wb{\featvec}_t^\transp\wb{\featkermatrix}_{t-1}\selrmatrix_{t-1}(\selrmatrix_{t-1}^\transp\wb{\featkermatrix}_{t-1}^\transp\wb{\featkermatrix}_{t-1}\selrmatrix_{t-1} + \alpha\bI)^{-1}\selrmatrix_{t-1}^\transp\wb{\featkermatrix}_{t-1}^\transp\wb{\featvec}_t)\\
&= \frac{1}{\alpha}(k_{t,t} - \wb{\bk}_{[t-1],t}^\transp\selrmatrix_{t-1}(\selrmatrix_{t-1}^\transp\wb{\kermatrix}_{t-1}\selrmatrix_{t-1} + \alpha\bI)^{-1}\selrmatrix_{t-1}^\transp\wb{\bk}_{[t-1],t}).
\end{align*}
For $\breve{y}_t$, we have
\begin{align*}
\breve{y}_{t} &= \featvec_{t}^\transp\wt{\bu}_{t} = \featvec_{t}^\transp\wt{\bA}_{t-1}^{-1}\wb{\featkermatrix}_{t-1}\wt{\bb}_{t-1}\\
&= \featvec_{t}^\transp(\wb{\featkermatrix}_{t-1}\selrmatrix_{t-1}\selrmatrix_{t-1}^\transp\wb{\featkermatrix}_{t-1}^\transp + \alpha\bI)^{-1}\wb{\featkermatrix}_{t-1}\wt{\bb}_{t-1}\\
&= \frac{1}{\alpha}\featvec_{t}^\transp\left(\bI - \wb{\featkermatrix}_{t-1}\selrmatrix_{t-1}(\selrmatrix_{t-1}^\transp\wb{\featkermatrix}_{t-1}^\transp\wb{\featkermatrix}_{t-1}\selrmatrix_{t-1} + \alpha\bI)^{-1}\selrmatrix_{t-1}^\transp\wb{\featkermatrix}_{t-1}^\transp\right)\wb{\featkermatrix}_{t-1}\wt{\bb}_{t-1}\\
&= \frac{1}{\alpha}\featvec_{t}^\transp\featkermatrix_{t-1}\bD_{t-1}\left(\wt{\bb}_{t-1} - \selrmatrix_{t-1}(\selrmatrix_{t-1}^\transp\wb{\kermatrix}_{t-1}\selrmatrix_{t-1} + \alpha\bI)^{-1}\selrmatrix_{t-1}^\transp\wb{\kermatrix}_{t-1}\wt{\bb}_{t-1}\right)\\
&= \frac{1}{\alpha}\bk_{[t-1],t}^\transp\bD_{t-1}\left(\wt{\bb}_{t-1} - \selrmatrix_{t-1}(\selrmatrix_{t-1}^\transp\wb{\kermatrix}_{t-1}\selrmatrix_{t-1} + \alpha\bI)^{-1}\selrmatrix_{t-1}^\transp\wb{\kermatrix}_{t-1}\wt{\bb}_{t-1}\right).
\end{align*}
\end{proof}

\begin{proof}[Proof of Theorem~\ref{thm:sketch-ons-main-thm}]
Since the only thing that changed is the formulation of the $\bA_t$ matrix,
the bound from Prop.\,\ref{prop:exact-ons-regret} still applies. In particular,
we have that the regret $\wt{R}_T$ of Alg.\,\ref{alg:ons-sketch} is bounded as
\begin{align*}
\wt{R}(\bw) \leq &\alpha\normsmall{\bw}_{\bA_0}^2
+ \sum\nolimits_{t=1}^T \bg_t^\transp\wt{\bA}_t^{-1}\bg_t +\sum\nolimits_{t=1}^T (\bw_t - \bw)^\transp(\wt{\bA}_t - \wt{\bA}_{t-1} - \sigma_t\bg_t\bg_t^\transp)(\bw_t - \bw).
\end{align*}
From Thm.\,\ref{thm:online-squeak}, we have that \onlinesqueak succeeds with high probability.
In particular, using the guarantees of the $\varepsilon$-accuracy (1), we can bound for the case $\eta_t = \sigma$ as
\begin{align*}
\bg_t^\transp\wt{\bA}_t^{-1}\bg_t
&= \frac{\eta_t}{\eta_t}\bg_t^\transp(\wb{\featkermatrix}_t\selrmatrix_t\selrmatrix_t^\transp\wb{\featkermatrix}_t^\transp + \alpha\bI)^{-1}\bg_t
= \frac{1}{\eta_t}\wb{\featvec}_t^\transp(\wb{\featkermatrix}_t\selrmatrix_t\selrmatrix_t^\transp\wb{\featkermatrix}_t^\transp + \alpha\bI)^{-1}\wb{\featvec}_t\\
&= \frac{1}{\eta_t}\frac{\wt{p}_{\min}}{\wt{p}_{\min}}\wb{\featvec}_t^\transp(\wb{\featkermatrix}_t\selrmatrix_t\selrmatrix_t^\transp\wb{\featkermatrix}_t^\transp + \alpha\bI)^{-1}\wb{\featvec}_t
= \frac{1}{\eta_t}\frac{1}{\wt{p}_{\min}}\wb{\featvec}_t^\transp\left(\frac{1}{\wt{p}_{\min}}\wb{\featkermatrix}_t\selrmatrix_t\selrmatrix_t^\transp\wb{\featkermatrix}_t^\transp + \alpha\bI\right)^{-1}\wb{\featvec}_t\\
&\leq \frac{1}{\eta_t}\frac{1}{\wt{p}_{\min}}\wb{\featvec}_t^\transp\left(\wb{\featkermatrix}_t\selmatrix_t\selmatrix_t^\transp\wb{\featkermatrix}_t^\transp + \alpha\bI\right)^{-1}\wb{\featvec}_t\\
&\leq \frac{1}{\wt{p}_{\min}\eta_t}\wb{\featvec}_t^\transp((1-\varepsilon)\wb{\featkermatrix}_t\wb{\featkermatrix}_t^\transp -\epsilon\alpha\bI + \alpha\bI)^{-1}\wb{\featvec}_t\\
&= \frac{1}{(1-\varepsilon)\sigma\wt{p}_{\min}}\wb{\featvec}_t^\transp(\wb{\featkermatrix}_t\wb{\featkermatrix}_t^\transp + \alpha\bI)^{-1}\wb{\featvec}_t = \frac{\wb{\tau}_{t,t}}{(1-\varepsilon)\sigma\wt{p}_{\min}}\CommaBin
\end{align*}
where in the first inequality we used the fact that the weight matrix $\bS_t$ contains weights such that $1/\sqrt{\wt{p}_{\min}} \geq 1/\sqrt{\wt{p}_t}$, in the second inequality we used the $\varepsilon$-accuracy, and finally, we used $\eta_t = \sigma$ and the definition of $\wb{\tau}_{t,t}$.
Therefore,
\begin{align*}
R_G = \sum\nolimits_{t=1}^T \bg_t^\transp\wt{\bA}_t^{-1}\bg_t
\leq \frac{1}{(1-\varepsilon)\sigma\wt{p}_{\min}}\sum\nolimits_{t=1}^T\wb{\tau}_{t,t}
\leq \frac{\wb{d}_{\text{onl}}(\alpha)}{(1-\varepsilon)\sigma\max\{\beta\wb{\tau}_{\min}, \gamma\}}\cdot
\end{align*}
To bound $R_D$, we have
\begin{align*}
\sum\nolimits_{t=1}^T (\bw_t - \bw)^\transp(\wt{\bA}_t - \wt{\bA}_{t-1} - \sigma_t\bg_t\bg_t^\transp)(\bw_t - \bw)
&= \sum\nolimits_{t=1}^T (\bw_t - \bw)^\transp\left(\eta_t z_t\bg_t\bg_t^\transp - \sigma_t\bg_t\bg_t^\transp\right)(\bw_t - \bw)\\
&\leq \sum\nolimits_{t=1}^T (\sigma - \sigma_t)(\bg_t^\transp(\bw_t - \bw))^2
 \leq 0.
\end{align*}
\end{proof}

\todomi{influence of $\gamma$ ? on the space complexity?}

\end{document}